\documentclass[11pt]{article}
\usepackage[margin=1in]{geometry}
\usepackage{preamble}
\title{Learning (Very) Simple Generative Models Is Hard}
\author{
    Sitan Chen\thanks{Email: \texttt{sitanc@berkeley.edu}} \\
    UC Berkeley
        \and 
    Jerry Li\thanks{Email: \texttt{jerrl@microsoft.com}} \\
    Microsoft Research
        \and
    Yuanzhi Li\thanks{Email: \texttt{yuanzhil@andrew.cmu.edu}} \\
    CMU
}

\DeclareMathOperator{\relu}{ReLU}
\DeclareMathOperator{\diag}{diag}
\newcommand{\op}{\mathsf{op}}
\newcommand{\STAT}{\mathrm{STAT}}
\newcommand{\VSTAT}{\mathrm{VSTAT}}
\newcommand{\calC}{\mathcal{C}}
\DeclareMathOperator{\dtv}{d_{\mathrm{TV}}}

\begin{document}

\maketitle

\begin{abstract}
Motivated by the recent empirical successes of deep generative models, we study the computational complexity of the following unsupervised learning problem. For an unknown neural network $F:\mathbb{R}^d\to\mathbb{R}^{d'}$, let $D$ be the distribution over $\mathbb{R}^{d'}$ given by pushing the standard Gaussian $\mathcal{N}(0,\textrm{Id}_d)$ through $F$. Given i.i.d. samples from $D$, the goal is to output \emph{any} distribution close to $D$ in statistical distance.

We show under the statistical query (SQ) model that no polynomial-time algorithm can solve this problem even when the output coordinates of $F$ are one-hidden-layer ReLU networks with $\log(d)$ neurons. Previously, the best lower bounds for this problem simply followed from lower bounds for \emph{supervised learning} and required at least two hidden layers and $\poly(d)$ neurons \cite{chen2022hardness,daniely2021local}.
    
The key ingredient in our proof is an ODE-based construction of a compactly supported, piecewise-linear function $f$ with polynomially-bounded slopes such that the pushforward of $\mathcal{N}(0,1)$ under $f$ matches all low-degree moments of $\mathcal{N}(0,1)$.\end{abstract}

\section{Introduction}
In recent years, deep generative models such as variational autoencoders, generative adversarial networks, and normalizing flows~\cite{goodfellow2014generative,kingma2013auto,rezende2015variational} have seen incredible success in modeling real world data.
These work by learning a parametric transformation (e.g. a neural network) of a simple distribution, usually a standard normal random variable, into a complex and high-dimensional one.
The learned distributions have been shown to be shockingly effective at modeling real world data.
The success of these generative models begs the following question: when is it possible to learn such a distribution?
Not only is this a very natural question from a learning-theoretic perspective, but understanding this may also lead to more direct methods to learn generative models for real data.

More formally, we consider the following problem.
Let $D$ be the unknown pushforward distribution over $\R^d$ given by $f(g)$, where $g \sim \calN(0, \Id)$ is a standard normal Gaussian, and $f$ is an unknown feed-forward neural network with non-linear (typically ReLU) activations.
Such distributions naturally arise as the output of many common deep generative models in practice.
The learner is given $n$ samples from $D$, and their goal is to output the description of some distribution which is close to $D$.

When $f$ is a one-layer network (i.e. of the form $f(g) = \relu(W g)$), there are efficient algorithms for learning the distribution~\cite{wu2019learning,lei2020sgd}.
However, this setting is unsatisfactory in many ways, as one-layer networks lack much of the complex structure that makes these generative models so appealing in practice.
Indeed, when the neural network only has a single layer, the resulting distribution is similar to a truncated Gaussian, and one can leverage techniques developed for learning from truncated samples.
Notably, this structure disappears even with two-layer neural networks.
Even in the two-layer case, despite significant interest, very little is known about how to learn $D$ efficiently. 

In fact, a recent line of work suggests that learning neural network pushforwards of Gaussians may be an inherently difficult computational task.
Recent results of~\cite{daniely2021local,chen2022hardness} show hardness of \emph{supervised} learning from labeled Gaussian examples under cryptographic asssumptions, and the latter also demonstrates hardness for all statistical query (SQ) algorithms (see Section~\ref{sec:related-work} for a more detailed description of related work).
These naturally imply hardness in the unsupervised setting (see Appendix~\ref{sec:hardness_supervised}). However, these lower bound constructions still have their downsides.
For one, all of these constructions require at least three layers (i.e. two hidden layers), and so leave open the possibility that efficient learning is possible when the neural network only has one hidden layer.
Additionally, the resulting neural networks in these constructions are quite complicated.
In particular, the size of the neural networks in these hard instances, that is, the number of hidden nonlinear activations for any output coordinate, must be polynomially large.
This begs the natural question:
\begin{center}
    {\it Can we learn pushforwards of Gaussians under one-hidden-layer neural networks of small size?}
\end{center}



\subsection{Our Results}
\label{sec:ourresults}
We demonstrate strong evidence that despite the simplicity of the setting, this learning task is already computationally intractable.
We show there is no polynomial-time statistical query (SQ) algorithm which can learn the distribution of $f(g)$, when $g \sim \calN(0, \Id)$ and each output coordinate of $f$ is a one-hidden-layer neural networks of \emph{logarithmic hidden size}.
We formally define the SQ model in Section~\ref{sec:prelims}; we note that it is well-known to capture almost all popular learning algorithms \cite{feldman2017statistical}.

\begin{theorem}[informal, see Theorem~\ref{thm:mainlbd}]\label{thm:informal}
For any $d > 0$, and any $C \geq 1$, there exists a family of one-hidden-layer neural networks $\calF$ from $\R^d$ to $\R^{d^C}$ so that the following properties hold.
For any $f \in \calF$, let $D_f$ denote the distribution of $f(g)$, for $g \sim \calN(0, \Id)$.
Then, we have that:
\begin{itemize}[itemsep=0pt,topsep=0pt,leftmargin=*]
    \item For all $f \in \calF$, $\dtv (\calN(0, \Id), D_f) = \Omega (1)$,\footnote{$\dtv$ denotes total variation distance. A lower bound for Wasserstein distance also holds, see Appendix~\ref{sec:wasserstein}.}
    \item Every output coordinate of $f$ is a sum of $O(\log d / \log \log d)$ ReLUs, with $\poly(d)$-bounded weights.
    \item Any SQ algorithm which can distinguish between $D_f$ and $\calN(0, \Id)$ with high probability for all $f \in \calF$ requires $d^{\omega (1)}$ time and/or samples.
\end{itemize}
\end{theorem}

\noindent In other words, there is a family of one-hidden-layer ReLU networks of logarithmic size whose corresponding pushforwards are statistically very far from Gaussian, yet no efficient SQ algorithm can distinguish them from a Gaussian. Note this implies hardness even of \emph{improperly} learning such pushforwards: not only is it hard to recover the parameters of the network or output a network close to the underlying distribution $D_f$, but it is hard to learn \emph{any} distribution close to $D_f$.

Since such networks are arguably some of the simplest neural networks with more than one layer, this suggests that learning even the most basic deep generative models may already be a very difficult task, at least without additional assumptions.
Still, this is by no means the last word in this direction.
Given the real world success of deep generative models, a natural and important direction is to identify natural conditions under which we can efficiently learn. We view our results as a first step towards understanding the computational landscape of this important learning problem, and our result \emph{provides evidence that (strong) assumptions need to be made on $f$ for the pushforward $f(g)$ to be efficiently learnable, even in very simple two-layer cases}.

\subsection{Our Techniques}

Like many recent SQ lower bounds, ours follows the general framework which was introduced in \cite{diakonikolas2017statistical} and builds on \cite{feldman2017statistical}. Here one considers the following ``non-Gaussian component analysis'' task. Let $D$ be a known, non-Gaussian distribution $D$ over $\R$. Given a unit vector $v$ in $d$ dimensions, let $P^D_v$ denote the distribution over $\R^d$ whose projection along $v$ is given by $D$ and whose projection in all directions orthogonal to $v$ is standard Gaussian. Given samples from some unknown distribution over $\R^d$, the goal is to decide whether the unknown distribution is $\calN(0,\Id_d)$ or $P^D_v$ for some $v$. \cite{diakonikolas2017statistical} showed that if $D$'s moments match those of $\calN(0,1)$ up to some degree $m$, then under mild conditions, any SQ algorithm for this task requires at least $d^{\Omega(m)}$ queries (Lemma~\ref{lem:generic}).

Suppose one could exhibit a one-hidden-layer ReLU network $f:\R^\ell\to\R$ such that the pushforward $D = f(\calN(0,\Id))$ satisfied such properties. Then we can realize $P^D_v$ as a pushforward as follows. Let $U$ be a rotation mapping the first standard basis vector in $\R^d$ to $v$. Then consider the function $F:\R^{\ell+d-1}\to\R^d$ mapping $z$ to $U\cdot (f(z_1,\ldots,z_\ell),z_{\ell+1},\ldots,z_{\ell+d-1})$. One can check that every output coordinate of $F$ is computed by a one-hidden-layer ReLU network with size essentially equal to that of $f$. By \cite{diakonikolas2017statistical}, we would immediately get the desired SQ lower bound.

The main challenge is thus to construct such a network whose pushforward matches the low-degree moments of $\calN(0,1)$. It is not hard to ensure the existence of such $f$ with essentially \emph{infinite} weights (Corollary~\ref{cor:step} and Lemma~\ref{lem:boxes}). It is much less clear whether this is possible with \emph{polynomially bounded} weights, and this is our primary technical contribution. We design and analyze a certain ODE which defines a one-parameter family of perturbations to $f$, such that the low-degree moments of the corresponding pushforwards remain unchanged over time. By evolving along this family over an inverse-polynomial time scale, we obtain a network with polynomially bounded weights whose pushforward matches the low-degree moments of $\calN(0,1)$. We defer the details to Section~\ref{sec:moment}.

\subsection{Related Work}
\label{sec:related-work}
A full literature review on the theory of learning deep generative models is beyond the scope of this paper (see e.g. the survey of ~\cite{gui2021review}). For conciseness we cover only the most relevant papers.

\paragraph{Upper bounds.}
In terms of upper bounds, much of the literature has focused on a different setting, where the goal is to understand when first order dynamics can learn toy generative models~\cite{feizi2017understanding,daskalakis2017training,gidel2019negative,lei2020sgd,allen2021forward,jelassi2022adam}, which are much simpler than the ones we consider here.
For learning pushforwards of Gaussians under neural networks with ReLU activations,  algorithms with provable guarantees are only known when the network has no hidden layers~\cite{wu2019learning,lei2020sgd}.
This is in contrast to the supervised setting, where fixed parameter tractable algorithms are known for learning ReLU networks of arbitrary depth~\cite{chen2022learningb}.

A different line of work seeks to find efficient learning algorithms when the activations are given by low degree polynomials~\cite{feizi2017understanding,li2020making,chen2022learning}.
Arguably the closest to our work is~\cite{chen2022learning}, which gives polynomial-time algorithms for learning low-degree polynomial transformations of Gaussians, in a smoothed setting.
It is a very interesting open question if similar smoothed assumptions can be leveraged to circumvent our lower bound when the activations are ReLU.
Unfortunately, these papers heavily leverage the nice moment structure of low-degree Gaussian polynomials, and it is unclear how their techniques can generalize to different activations.

\paragraph{Lower bounds.} Much of the literature on lower bounds for learning neural networks has focused on the supervised setting, where a learner is given labeled examples $(x, f(x))$, and the goal is to output a good predictor.
There are many lower bounds known in the distribution-free setting~\cite{blum1992training,vu1998infeasibility,klivans2009cryptographic,livni2014computational,daniely2020hardness}, however, these do not transfer over to our (unsupervised) setting.
When $x$ is Gaussian, the aforementioned work of~\cite{chen2022hardness} derives hardness for learning two-hidden-layer networks with polynomial size for all SQ algorithms, as well as under cryptographic assumptions (see also \cite{daniely2021local}).
It is not hard to show (see Appendix~\ref{sec:hardness_supervised}) that this lower bound immediately implies a lower bound for the unsupervised problem.
In the supervised setting, lower bounds are also known against restricted families of SQ~\cite{goel2020superpolynomial,diakonikolas2020algorithms,song2017complexity}, when there are adversarially noisy labels~\cite{klivans2014embedding,diakonikolas2020near,goel2020statistical,song2021cryptographic}, and in discrete settings~\cite{valiant1984theory,kharitonov1995cryptographic,angluin1995won,feldman2009power,cohen2015aggregate,das2020learnability,agarwal2021deep}, but to our knowledge, these results do not transfer to our setting.

The literature on lower bounds for the unsupervised problem we consider here is much sparser.
Besides \cite{daniely2021local,chen2022hardness}, we also mention the recent work of~\cite{chen2022minimax} that studies whether achieving small Wasserstein GAN loss implies distribution learning. A corollary of their results is cryptographic hardness for learning pushforwards of Gaussians under networks with constant depth and polynomial size, but only when the learner is given by a Lipschitz ReLU network discriminator.
However, this does not rule out efficient algorithms which do not output such Lipschitz discriminators.

Finally, we remark that the family of hard distributions we construct can be thought of as a close cousin of the ``parallel pancakes'' construction of \cite{diakonikolas2017statistical}. This and slight modifications thereof are mixtures of Gaussians which are known to be computationally hard to known both in the SQ model \cite{diakonikolas2017statistical,bubeck2019adversarial} and under cryptographic assumptions \cite{bruna2021continuous,gupte2022continuous}.

\paragraph{SQ lower bounds via ODEs.}  We remark that in a very different context, \cite{diakonikolas2020near} also used an ODE to design a moment-matching construction. While our approach draws inspiration from theirs, an important difference is that they use their ODE as a ``size reduction'' trick to construct a step function $f:\R\to\brc{\pm 1}$ with a small number of linear pieces such that $\E[g\sim\calN(0,1)]{f(g)g^k} = 0$ for all small $k$, while we use our ODE as a ``weight reduction'' trick to construct a continuous neural network $f:\R\to\R$ with bounded weights such that $\E[g\sim\calN(0,1)]{f(g)^k} = \E[g\sim\calN(0,1]{g^k}$. The form of the moments $\E[g\sim\calN(0,1)]{f(g)g^k}$ they consider is simpler than in our setting, and while they essentially run their ODE to singularity and use non-quantitative facts like the invertibility of a certain Jacobian, we only run our ODE for a finite horizon and need to carefully control the condition number of the Jacobian arising in our setting over this horizon (e.g. Lemma~\ref{lem:latercond}).

\section{Technical Preliminaries}
\label{sec:prelims}

\paragraph{Notation.} We freely abuse notation and use the same symbols to denote probability distributions, their laws, and their density functions. Given a distribution $A$ over a domain $\Omega$ and a function $f:\Omega\to\Omega'$, we let $f(A)$ denote the \emph{pushforward} of $A$ through $f$, that is, the distribution $A'$ of the random variable $f(z)$ for $z\sim A$. Let $p\star q$ denote the convolution of $p$ and $q$. Also, we use $\norm{\cdot}_p$ to denote $\ell^p$ norm, omitting the subscript when $p = 2$. $\sigma_{\min}(\cdot)$ denotes minimum singular value.

\paragraph{Neural networks.} Define $\relu(z) \triangleq \max(0,z)$.

\begin{definition}[One-hidden-layer ReLU networks]\label{def:networks}
    We say that $g: \R^d\to\R$ is a \emph{one-hidden-layer ReLU network with size $S$ and $W$-bounded weights} if there exist $w_1,\ldots,w_S\in\R^d$,  $b_1,\ldots,b_S\in\R$, and $s_1,\ldots,s_S\in\brc{\pm 1}$ for which
    \begin{equation}
        g(x) = \sum^S_{i=1} s_i\relu(\iprod{w_i,x} + b_i) \ \ \forall \ x\in\R^d.
    \end{equation}
    and $\norm{w_i}, |b_i| \le W$ for all $i\in S$.
    
    Given $f: \R^d\to\R^{d'}$ whose output coordinates are of this form, together with a distribution $A$ over $\R^d$, we say that $f(A)$ is a \emph{one-hidden-layer ReLU network pushforward of $A$ with size $S$ and $W$-bounded weights.}
\end{definition}

\paragraph{Statistical query lower bounds.} Here we review standard concepts pertaining to establishing statistical query lower bounds for unsupervised learning problems, as developed in \cite{feldman2017statistical}.

\begin{definition}[Distributional search problems]
    Let $\calD$ be a set of probability distributions, let $\calF$ be a set of \emph{solutions}, and let $\mathcal{Z}: \calD\to 2^{\calF}$ be a map that takes any $D\in\calD$ to a subset of $\calF$ corresponding to the valid solutions for $D$. We say that $\mathcal{Z}$ specifies a \emph{distributional search problem over $\calD$ and $\calF$}: given oracle access to an unknown $D\in\calD$, the goal of the learner is to output a valid solution from $\mathcal{Z}(D)$.
\end{definition}

\begin{definition}[Statistical query oracles]
    Given a distribution $D$ over $\R^d$ and parameters $\tau, t > 0$, a \emph{$\mathrm{STAT}(\tau)$ oracle} takes in any query of the form $f: \R^d\to[-1,1]$ and outputs a value from the interval $[\E[x\sim D]{f(x)} - \tau, \E[x\sim D]{f(x)} + \tau]$, while a \emph{$\mathrm{VSTAT}(t)$ oracle} takes in any query of the form $f: \R^d\to[0,1]$ and outputs a value from the interval $[\E[x\sim D]{f(x)} - \tau, \E[x\sim D]{f(x)} + \tau]$ for $\tau = \max(1/t, \sqrt{\Var[x\sim D]{f(x)} / t})$.
\end{definition}

\begin{definition}[Pairwise correlation]
    Given distributions $p,q$ over a domain $\Omega$ which are absolutely continuous with respect to a distribution $r$ over $\Omega$, we let $\chi^2_r(p,q)$ denote the \emph{pairwise correlation}, that is
    \begin{equation}
        \chi^2_r(p,q) \triangleq \int_\Omega p(x)q(x)/r(x) \, \mathrm{d}x - 1.
    \end{equation}
    Note that when $p = q$, this is simply the chi-squared divergence between $p$ and $r$.
    
    We say that a set of $m$ distributions $\calD = \brc{D_1,\ldots,D_m}$ is $(\gamma,\beta)$-correlated relative to a distribution $\mu$ over $\R^d$ if
    \begin{equation}
        |\chi_\mu(D_i,D_j)| \le \begin{cases}
            \gamma & \text{if} \ i\neq j \\
            \beta & \text{if} \ i = j
        \end{cases}.
    \end{equation}
\end{definition}

\begin{definition}[Statistical dimension]\label{def:sqdim}
    Let $\beta,\gamma>0$, let $\mathcal{Z}$ be a distributional search problem over distributions $\calD$ and solutions $\calF$, and let $N$ be the largest integer for which there exists a distribution $\mu$ and a finite subset $\calD_\mu\subseteq \calD$ such that for any $f\in\calF$, $\calD_f \triangleq \calD_\mu \backslash\mathcal{Z}^{-1}(f)$ is $(\gamma,\beta)$-correlated relative to $\mu$ and $|\calD_f| \ge N$. We say that the \emph{statistical dimension} with pairwise correlations $(\gamma,\beta)$ of $\mathcal{Z}$ is $N$ and denote it by $\mathrm{SD}(\mathcal{Z},\gamma,\beta)$.
\end{definition}

\begin{lemma}[Corollary 3.12 from \cite{feldman2017statistical}]\label{lem:feldman}
    Let $\mathcal{Z}$ be a distributional search problem over distributions $\calD$ and solutions $\calF$. For $\gamma,\beta$, if $N = \mathrm{SD}(\mathcal{Z},\gamma,\beta)$, then any statistical query algorithm for $\mathcal{Z}$ requires at least $N\gamma/(\beta-\gamma)$ queries to $\mathrm{STAT}(\sqrt{2\gamma})$ or $\mathrm{VSTAT}(1/(6\gamma))$.
\end{lemma}

\paragraph{Gaussians and truncated Gaussians.} Henceforth $\E[g]{\cdot}$ will always denote $\E[g\sim\calN(0,\Id)]{\cdot}$. Let $\gamma_{\sigma^2}(x)\triangleq \frac{1}{\sigma\sqrt{2\pi}} e^{-x^2/(2\sigma^2)}$. Given $S\subset\R$, we will use $\gamma_{\sigma^2}(S)$ to denote $\int^\infty_{-\infty} \gamma(x)\cdot \bone{x\in S}\, \mathrm{d}x$. When $\sigma = 1$, we will omit the subscript $\sigma^2$. We will also use $\gamma^{(d)}(x)$ to denote the density of $\calN(0,\Id_d)$.

Given $m,i\in\mathbb{N}$, let $m^{\Downarrow i} \triangleq m(m - 2)\cdots (m - 2i + 2)$. Also let $m^{\Downarrow 0} = 1$. With this notation, we have the following expression for the moments of a truncated Gaussian.

\begin{lemma}
    For any $k\in\mathbb{N}$, define the polynomial
    \begin{equation}
        p_k(x) \triangleq \sum^{\floor{(k-1)/2}}_{i=0} (k-1)^{\Downarrow i} x^{k - 1 - 2i}. \label{eq:evenoddpoly}
    \end{equation}
    For any $a\le b$,
    \begin{equation}
        \E[g]{g^k\cdot \bone{a\le g \le b}} = \begin{cases}
            (k-1)!!\cdot \gamma([a,b]) - (p_k(b)\gamma(b) - p_k(a)\gamma(a)) & \text{if} \ k \ \text{even} \\
            - (p_k(b)\gamma(b) - p_k(a)\gamma(b)) & \text{if} \ k \ \text{odd}
        \end{cases}
    \end{equation}
\end{lemma}

\begin{corollary}\label{cor:shift_moment}
    For any $c,d\in\R$ and $k\in\mathbb{N}$ even,
    \begin{equation}
        \E[g]{(cg+d)^k\cdot \bone{a \le g \le b}} = \sum^k_{i=0 \ \text{even}} \binom{k}{i}c^id^{k-i}(k-1)!!\gamma([a,b]) - \sum^k_{i=0} \binom{k}{i}c^id^{k-i}(p_i(b)\gamma(b) - p_i(a)\gamma(a)).
    \end{equation}
\end{corollary}

\paragraph{Hidden direction distribution.} Given a distribution $D$ over $\R$ and $v\in\S^{d-1}$, let $P^D_{v}$ denote the distribution over $\R^d$ with density
\begin{equation}
    P^D_{v}(x) = D(\iprod{v,x})\cdot \gamma^{(d-1)}(x - \iprod{v,x}v),
\end{equation}
that is the distribution which is given by $D$ in the direction $v$ and is given by $\calN(0,\Id - vv^{\top})$ orthogonal to $v$.

\paragraph{Miscellaneous technical facts.} 

\begin{fact}\label{fact:tvequiv}
    Given two distributions $p,q$ over a domain $\Omega$, $d_{\mathrm{TV}}(p,q) = 1 - \int_\Omega \min(p(x), q(x)) \, \mathrm{d}x$.
\end{fact}

\begin{fact}[\cite{gordon2020sparse}] \label{fact:vandermonde}
    If $V\in\R^{n\times n}$ is a Vandermonde matrix with nodes $z_1,\ldots,z_n$, that is, $V_{i,j} = z_j^{i-1}$, and $\brc{z_i}$ are $\zeta$-separated, then $\sigma_{\min}(V) \ge \frac{1}{n}\cdot \Omega(\zeta)^{n-1}$.
\end{fact}

\begin{theorem}[Peano's existence theorem, see e.g. Theorem 2.1 from \cite{hartman2002ordinary}]\label{thm:peano}
    For $T,r>0$ and $y_0\in\R^n$, let $B\subset\R\times \R^n$ be the parallelepiped consisting of $(t,y)$ for which $0 \le t \le T$ and $\norm{y - y_0}_{\infty} \le r$. If $f: B\to\R$ is continuous and satisfies $|f(t,y)| \le M$ for all $(t,y)\in B$, then the initial value problem
    \begin{equation}
        y'(t) = f(t,y) \qquad \text{and} \qquad y(0) = y_0
    \end{equation}
    has a solution over $t\in[0,\min(T,r/M)]$.
\end{theorem}

\section{Statistical Query Lower Bound}

In this section we prove our main theorem:

\begin{theorem}\label{thm:mainlbd}
    Let $d\in\mathbb{N}$ be sufficiently large. Any SQ algorithm which, given SQ access to an arbitrary one-hidden-layer ReLU network pushforward of $\calN(0,\Id_d)$ of size $O(\log d / \log\log d)$ with $\poly(d)$-bounded weights, outputs a distribution which is $O(1)$-close in $\dtv(\cdot)$ must make at least $d^{\Omega(\log d / \log\log d)}$ queries to either $\STAT(\tau)$ or $\VSTAT(1/\tau^2)$ for $\tau = d^{-\Omega(\log d / \log\log d)}$.
\end{theorem}

\noindent Our proof will invoke the following key technical result whose proof we defer to Section~\ref{sec:moment}. Roughly, it exhibits a two-dimensional one-hidden-layer ReLU network $f:\R\to\R$ with bounded weights under which the pushforward of $\calN(0,1)$ matches the low-degree moments of $\calN(0,1)$ to arbitrary precision, in addition to some other technical conditions that we need to formally establish our statistical query lower bound:

\begin{theorem}\label{thm:moment}
    Fix any odd $m$ and $\nu, \sigma < 1$. There is a one-hidden-layer ReLU network $f^*:\R^2\to\R$ of size $O(m)$ with weights at most $m^{O(m)}$ for which the pushforward $D\triangleq f^*(\calN(0,\Id))$ satisfies
    \begin{enumerate}[leftmargin=*]
        \item $|\E[x\sim D]{x^k} - \E[g\sim\calN(0,1)]{g^k}| < \nu$ for all $k = 1,\ldots,m$
        \item $\chi^2(D,\calN(0,1)) \le \exp(O(m)) / \sigma$
        \item $\dtv(P^D_v, P^D_{v'}) \ge 1 - 2\sigma\log(1/\sigma) - m^{-\Omega(m)}$ for any $v,v'\in\S^{d-1}$ satisfying $|\iprod{v,v'}| \ge 1/2$.
    \end{enumerate}
\end{theorem}

\noindent The rest of the proof of our lower bound will then follow the framework introduced in \cite{diakonikolas2017statistical} and subsequently generalized in \cite{diakonikolas2020hardness}. Specifically, we will use the following two lemmas from these works:

\begin{lemma}[Lemma 3.5 from \cite{diakonikolas2020hardness}]\label{lem:correlation}
    There is an absolute constant $c > 0$ such that the following holds. Let $m\in\mathbb{N}$ and $\nu > 0$. If a distribution $D$ over $\R$ is such that 1) $\chi^2(D,\calN(0,1))$ is finite, and 2) $|\E[x\sim D]{x^k} - \E[g\sim\calN(0,1)]{g^k}| \le \nu$ for all $k = 1,\ldots, m$, then for all $v,v'\in\S^{d-1}$ for which $|\iprod{v,v'}| < c$,
    \begin{equation}
        |\chi^2_{\calN(0,\Id_d)}(P^D_v, P^D_{v'})| \le |\iprod{v,v'}|^{m+1} \chi^2(D,\calN(0,1)) + \nu^2.
    \end{equation}
\end{lemma}

\begin{fact}[Lemma 3.7 from \cite{diakonikolas2017statistical}]\label{fact:packing}
    For any constant $0 < C < 1/2$, there exists a set $S$ of $2^{d^C}$ unit vectors in $\S^{d-1}$ such that any pair of distinct $u,v\in S$ satisfies $|\iprod{u,v}| < d^{C-1/2}$.
\end{fact}

\noindent These can be used to prove the following generic statistical query lower bound:

\begin{lemma}\label{lem:generic}
    Let $m\in\mathbb{N}$ and $0 < C < 1/2$. Let $D$ be a distribution over $\R$ such that 1) $\chi^2(D,\calN(0,1))$ is finite, and 2) $|\E[x\sim D]{x^k} - \E[g\sim\calN(0,1)]{g^k}| \le \Omega(d)^{-(m+1)(1/4-C/2)}\sqrt{\chi^2(D,\calN(0,1))}$ for all $k = 1,\ldots, m$.
    
    Consider the set of distributions $\brc{P^D_v}_{v\in\S^{d-1}}$ for $d\ge m^{\Omega(1/C)}$. If there is some $\epsilon > 0$ for which $\tvd(P^D_v, P^D_{v'}) > 2\epsilon$ whenever $|\iprod{v,v'}| \le 1/2$, then any SQ algorithm which, given SQ access to $P^D_v$ for an unknown $v\in\S^{d-1}$, outputs a hypothesis $Q$ with $\tvd(Q,P^D_v) \le \epsilon$ needs at least $d^{m+1}$ queries to $\mathrm{STAT}(\tau)$ or to $\mathrm{VSTAT}(1/\tau^2)$ for $\tau\triangleq O(d)^{-(m+1)(1/4 - C/2)}\cdot \sqrt{\chi^2(D,\calN(0,1))}$.
\end{lemma}

\begin{proof}
    Let $S$ be the set of $2^{d^C}$ unit vectors from Fact~\ref{fact:packing}. In the notation of Definition~\ref{def:sqdim}, take $\mu = \calN(0,\Id)$ and let $\calD_\mu \triangleq \brc{P^D_v}_{v\in S}$. By Lemma~\ref{lem:correlation}, for any distinct $v,v'\in S$ we have
    \begin{equation}
        \chi_\mu(P^D_v, P^D_{v'}) \le |\iprod{v,v'}|^{m+1}\chi^2(D,\calN(0,1)) + O(\tau^2) \le \Omega(d)^{-(m+1)(1/2-C)}\chi^2(D,\calN(0,1)).
    \end{equation}
    On the other hand, if $v = v'\in S$, then $\chi_\mu(P^D_v, P^D_v) = \chi^2(D,\calN(0,1)) + O(\tau^2) \le 2\chi^2(D,\calN(0,1))$. So for
    \begin{equation}
        \gamma\triangleq \Omega(d)^{-(m+1)(1/2-C)}\chi^2(D,\calN(0,1)) \qquad \text{and} \qquad \beta \triangleq 2\chi^2(D,\calN(0,1)),
    \end{equation}
    $\calD_\mu$ is $(\gamma,\beta)$-correlated with respect to $\mu$.
    
    Consider the distributional search problem $\mathcal{Z}$ mapping any distribution $P^D_v$ to the set of probability distributions which are $\epsilon$-close in total variation distance to $P^D_v$. Because $\tvd(P_v,P_{v'}) > 2\epsilon$ for distinct $v,v'\in S$, for any distribution $f$ over $\R^d$ we have that $|\mathcal{Z}^{-1}(f)| \le 1$. We conclude that $\mathrm{SD}(\mathcal{Z},\gamma,\beta) \ge 2^{\Omega(d^C)}$. By Lemma~\ref{lem:feldman}, we conclude that any SQ algorithm for $\mathcal{Z}$ requires at least $2^{\Omega(d^C)}d^{-(m+1)(1/2-C)}$ calls to either $\mathrm{STAT}(\tau)$ or $\mathrm{VSTAT}(1/\tau^2)$. Note that because we are assuming that $d \ge m^{\Omega(1/C)}$, we have $2^{\Omega(d^{C/2})} \ge d^{m+1}$, so the total number of required queries is at least $2^{\Omega(d^{C/2})} \ge d^{m+1}$ as claimed.
\end{proof}

\noindent We are now ready to prove Theorem~\ref{thm:mainlbd}:

\begin{proof}[Proof of Theorem~\ref{thm:mainlbd}]
    By Theorem~\ref{thm:moment} applied with odd $m\in\mathbb{N}$ larger than some absolute constant and with $\sigma$ a sufficiently small absolute constant, there exists a distribution $D = f^*(\calN(0,\Id_2))$ over $\R$ for $f^*:\R^2\to\R$ of size $O(m)$ with $m^{O(m)}$-bouned weights satisfying the hypotheses of Lemma~\ref{lem:generic} for $\epsilon = 0.49$, and $\chi^2(D,\calN(0,1)) \le \exp(O(m))$. As long as $m \le d^{O(C)}$, we conclude that an SQ algorithm for learning any distribution from $\brc{P^D_v}_{v\in\S^{d-1}}$ to total variation distance $1/4$ must make at least $d^{m+1}$ queries to $\STAT(\tau)$ or $\VSTAT(1/\tau^2)$ for $\tau \triangleq O(d)^{-(m+1)(1/4-C/2)}\cdot \exp(O(m))$. By taking $m = \Theta(\log d/ \log\log d)$, we ensure that $m^{O(m)} \le \poly(d)$. By taking $C$ in Lemma~\ref{lem:generic} to be $C = 1/4$, we obtain the desired lower bound.
    
    The proof of the theorem is complete upon noting that any distribution $P^D_v$ can be implemented as a pushforward of $\calN(0,\Id_{d+1})$ under a one-hidden-layer ReLU network $F_v: \R^{d+1}\to\R^d$ of size $O(\log d / \log\log d)$ with $\poly(d)$-bounded weights. Let $U\in O(d)$ be a rotation mapping the first standard basis vector in $\R^d$ to $v$. Then for $F_v(z_1,\ldots,z_{d+1}) \triangleq U(f^*(z_1,z_2),z_3,\ldots,z_{d+1})$ we have that $F_v(\calN(0,\Id_{d+1})) = P^D_v$ as desired. Furthermore, note that every output coordinate of $F_v(z_1,\ldots,z_{d+1})$ is a one-hidden-layer ReLU network of the form $\alpha f^*(z_1,z_2) + \iprod{u,(z_3,\ldots,z_{d+1})}$ for some vector $(\alpha,u)\in\R^d$. Note that the size of this network is two plus that of $f^*$, and its weights are also upper bounded by $\poly(d)$, so $F_v$'s output coordinates are of size $O(\log d / \log\log d)$ as desired.
\end{proof}

\begin{remark}
    Theorem~\ref{thm:informal} was stated with output dimension polynomially bigger than input dimension, whereas in our construction, the output dimension ($d$) is less than the input dimension ($d+1$). One can get the former by a padding argument (i.e. by duplicating output coordinates) to give a generator with arbitrarily large polynomial stretch and such that the $d^{\log d / \log\log d}$ lower bound still applies.
\end{remark}

\section{Moment-Matching Construction}
\label{sec:moment}

In this section we prove Theorem~\ref{thm:moment}, the main technical ingredient in the proof of Theorem~\ref{thm:mainlbd}.

\subsection{Moment-Matching With Unbounded Weights}


In this section, we make the simple initial observation that for one-hidden-layer networks with \emph{unbounded weights}, it is easy to construct networks such that the pushforward of $\calN(0,1)$ under these networks matches the moments of $\calN(0,1)$ to arbitrary precision. The starting point for this observation is the following well-known moment-matching construction:

\begin{lemma}[Lemma 4.3 from \cite{diakonikolas2017statistical}]\label{lem:dks}
    For any $m\in\mathbb{N}$, there exist weights $\lambda_1,\ldots,\lambda_m\ge 0$ and points $h_1,\ldots,h_m\in\R$ for which 
    \begin{enumerate}[leftmargin=*]
        \item (Moments match) $\sum^m_{i=1} \lambda_i h_i^k = \E[g]{g^k}$ for all $k = 0,\ldots,2m-1$. \label{dks:moment}
        \item (Points symmetric about origin) $h_1 \le \cdots \le h_m$ and $h_i = -h_{m-i+1}$ for all $1\le i\le m$. \label{dks:sortedh}
        \item (Weights symmetric) $\lambda_1 \le \cdots \le \lambda_{\ceil{m/2}}$ and $\lambda_i = \lambda_{m-i+1}$. \label{dks:sortedlam}
        \item (Points bounded and separated) $\Omega(1/\sqrt{m}) \le |h_i| \le O(\sqrt{m})$ for all $1\le i \le m$ and $\brc{h_i}$ are $\Omega(1/\sqrt{m})$-separated. \label{dks:bounded}
        \item (Weights not too small) $\min_i \lambda_i \ge e^{-cm}$ for an absolute constant $c > 0$. \label{dks:mix}
        \item (Central point and weight) If $m$ is odd, then $h_{(m+1)/2} = 0$ and $\lambda_{(m+1)/2} = \Theta(1/\sqrt{m})$. \label{dks:middle}
    \end{enumerate}
\end{lemma}

\noindent This immediately implies that there exists a \emph{discontinuous} piecewise linear function $f:\R\to\R$ for which the pushforward $f(\calN(0,1))$ matches the low-degree moments of $\calN(0,1)$:

\begin{corollary}\label{cor:step}
    For any $m\in\mathbb{N}$, there is a partition of $\R$ into disjoint intervals $I_1,\ldots,I_m$, along with a choice of scalars $h_1,\ldots,h_m$, such that the step function $f:\R\to\R$ given by $f(z) = \sum^m_{i=1} h_i \cdot \bone{z\in I_i}$ satisfies $\E[x\sim f(\calN(0,1))]{x^k} = \E[g\sim\calN(0,1)]{g^k}$ for all $k = 0,\ldots,2m - 1$.
\end{corollary}

\begin{proof}
    Let $\lambda_1,\ldots,\lambda_m,h_1,\ldots,h_m$ be as in Lemma~\ref{lem:dks}. As $\sum_i \lambda_i = 1$, there is a partition of $\R$ into disjoint intervals $I_1,\ldots,I_m$ for which $\gamma(I_i) = \lambda_i$ for all $1\le i \le m$. Note that $\E[x\sim f(\calN(0,1))]{x^k} = \sum_i \lambda_i h_i^k$, so the claim follows by Part~\ref{dks:moment} of Lemma~\ref{lem:dks}.
\end{proof}

\noindent By infinitesimally perturbing the step function $f$ in Corollary~\ref{cor:step}, we can ensure that $f(\calN(0,1))$ still \emph{approximately} matches the low-degree moments of $\calN(0,1)$ to arbitrary precision and that the linear pieces of $f$ have finite slopes, though some slopes will now be arbitrarily large. Such a function $f$ can thus be represented as a one-hidden-layer ReLU network, but the issue is that the weights of this network will be arbitrarily large. The key technical challenge that we overcome in this section is to design a more careful way of perturbing $f$ so that the resulting piecewise linear function has \emph{polynomially bounded} slopes yet is such that $f(\calN(0,1))$ matches the low-degree moments of $\calN(0,1)$.

\subsection{Bump Construction}

Before we describe our perturbation scheme, we make a slight modification to the construction in Corollary~\ref{cor:step}. In place of a step function, we will consider a certain sum of \emph{bump functions}.

\begin{definition}[Bump functions]
    Given $w,\epsilon > 0$ and $h\in\R$, define $T^{w,h,\epsilon}: \R\to\R$ by
    \begin{equation}
        T^{w,h,\epsilon}(z) = \begin{cases}
            \frac{h}{\epsilon}(z + \epsilon + w) & \text{if} \ z \in [-\epsilon-w,-w] \\
            h & \text{if} \ z \in [-w,w] \\
            -\frac{h}{\epsilon}(z - \epsilon - w) & \text{if} \ z\in [w,\epsilon+w] \\
            0 & \text{otherwise}.
        \end{cases}
    \end{equation}
    Given $c\in\R$, define $T^{w,h,\epsilon}_c:\R\to\R$ by $T^{w,h,\epsilon}_c(z) = T^{w,h,\epsilon}(z - c)$.
\end{definition}

\noindent As $T^{w,h,\epsilon}_c$ is continuous piecewise-linear, it can be represented as a one-hidden-layer ReLU network. The following elementary fact makes explicit the relation between the parameters of a bump function and the parameters of the corresponding network implementing it.

\begin{fact}\label{fact:implementbox}
    Given $w,\epsilon>0$ and $h,c\in\R$, $T^{w,h,\epsilon}_c$ can be implemented as a one-hidden-layer ReLU network with size $4$ and $W$-bounded weights for $W \le \frac{h}{\epsilon}\max(1,|c| + \epsilon+w)$.
\end{fact}

\begin{proof}
    For all $z\in\R$, $T^{w,h,\epsilon}_c(z)$ is equal to
    \begin{align}
        & \frac{h}{\epsilon}\left(\relu(z - c + \epsilon + w) - \relu(z - c + w) - \relu(z - c - w) + \relu(z - c - \epsilon - w)\right).\qedhere
    \end{align}
\end{proof}

\begin{figure}[h]
    \centering
    \includegraphics[width=0.9\textwidth]{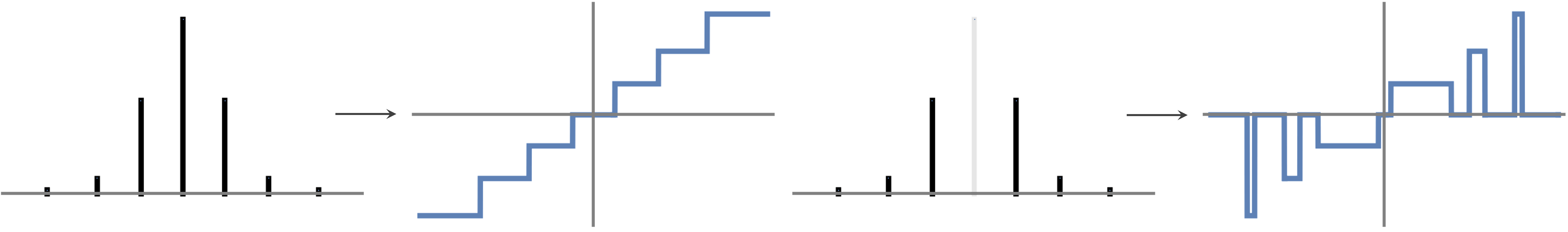}
    \caption{Left: construction from Lemma~\ref{lem:dks} gives rise to step function in Corollary~\ref{cor:step}. Right: removing central spike gives rise to sum of bumps in Lemma~\ref{lem:boxes}}
    \label{fig:piecewise}
\end{figure}

\noindent We now show how to replace the step function in Corollary~\ref{cor:step} with a sum of bump functions for which $\epsilon = 0$. As this new function will be the basis for the perturbation scheme we introduce in the next section, we also provide some quantitative bounds for its parameters:

\begin{lemma}\label{lem:boxes}
    For any odd $m\in\mathbb{N}$ and $0 < \nu < 1$, there exist centers $c_1\le\cdots \le c_{m-1}$, widths $w_1,\ldots,w_{m-1} > 0$, heights $h_1\le\cdots\le h_{m-1}\in\R$, and parameter $\overline{\epsilon} > 0$ for which the following holds. Define the function $f:\R\to\R$ by
    \begin{equation}
        f(z)\triangleq \sum^{m-1}_{i=1} T^{w_i,h_i,\epsilon}_{c_i}(z) \label{eq:infinitebump_def}
    \end{equation}
    for any $0 \le \epsilon < \overline{\epsilon}$ (see Figure~\ref{fig:piecewise}).
    Then $f$ satisfies
    \begin{enumerate}[leftmargin=*]
        \item (Bumps are well-separated) For all $1 \le i < m - 1$, $c_i + m^{-3/2} \le c_{i+1}$. \label{boxes:separated}
        \item (Moments match) $\E[x\sim f(\calN(0,1))]{x^k} = \E[g]{g^k}$ for all $k = 1,\ldots,2m-1$. \label{boxes:moment}
        \item (Symmetricity) $w_i = w_{m-i}$, $h_i = -h_{m-i}$, and $c_i = -c_{m-i}$ for all $1 \le i < m$. \label{boxes:sym}
        \item (Bounded and separated heights) $\Omega(1/\sqrt{m}) \le |h_i| \le O(\sqrt{m})$ for all $1 \le i < m$, and $\brc{h_i}$ are $\Omega(1/\sqrt{m})$-separated. \label{boxes:bounded}
        \item (Intervals not too thin) $\min_i \gamma([c_i - w_i,c_i+w_i]) \ge e^{-cm}$ for an absolute constant $c > 0$. \label{boxes:thin}
        \item (Bounded endpoints) $|c_i| + w_i \le O(\log m)$ for all $1 \le i < m$. \label{boxes:endpoints}
    \end{enumerate}
\end{lemma}

\noindent To prove Lemma~\ref{lem:boxes}, we will need the following modification of the construction in Lemma~\ref{lem:dks}.

\begin{lemma}\label{lem:dks2}
    For any odd $m\in\mathbb{N}$, there exist weights $\lambda_1,\ldots,\lambda_{m-1} \ge 0$ and points $h_1,\ldots,h_{m-1}\in\R$ for which
    \begin{enumerate}[leftmargin=*]
        \item (Sum of weights bounded away from 1) $\sum^{m-1}_{i=1} \lambda_i = 1 - \Theta(1/\sqrt{m})$.
        \item (Moments match) $|\sum^{m-1}_{i=1} \lambda_i h^k_i - \E[g]{g^k}| < \nu$ for all $k = 1,\ldots,2m-1$.
        \item (Points symmetric about origin) $h_1 \le \cdots \le h_{m-1}$ and $h_i = -h_{m-i}$ for all $1 \le i < m$.
        \item (Weights symmetric) $\lambda_1 \le \cdots \le \lambda_{(m-1)/2}$ and $\lambda_i = \lambda_{m-i}$.
        \item (Points bounded and separated) $\Omega(1/\sqrt{m}) \le |h_i| \le O(\sqrt{m})$ for all $1 \le i < m$, and $\brc{h_i}$ are $\Omega(1/\sqrt{m})$-separated.
        \item (Weights not too small) $\min_i \lambda_i \ge e^{-cm}$ for an absolute constant $c > 0$. \label{dks2:nottoosmall}
    \end{enumerate}
\end{lemma}

\begin{proof}
    Because $m$ is odd, we can take the weights and points to be given by Lemma~\ref{lem:dks} and remove the $(m+1)/2$-th weight and point\--- recall that the $(m+1)/2$-th point is 0 and thus does not contribute to $\sum_i \lambda_i h^k_i$. The fact that $\sum^{m-1}_{i=1} \lambda_i = 1 - \Theta(1/\sqrt{m})$ then follows from the fact that the $(m+1)/2$-th weight from Lemma~\ref{lem:dks} is of order $\Theta(1/\sqrt{m})$ by Part~\ref{dks:middle} of Lemma~\ref{lem:dks}. The remaining parts of the lemma follow by the corresponding parts of Lemma~\ref{lem:dks}.
\end{proof}

\begin{proof}[Proof of Lemma~\ref{lem:boxes}]
    Let $\lambda_1,\ldots,\lambda_{m-1},h_1,\ldots,h_{m-1}$ be as in Lemma~\ref{lem:dks2}. As $\sum_i\lambda_i = 1 - \Theta(1/\sqrt{m})$, we claim there exist intervals $I_1,\ldots,I_{m-1}$ such that for any $i < j$, all points in $I_i$ are strictly smaller than all points in $I_j$, such that $\gamma(I_i) = \lambda_i$ for all $i$, and such that the right endpoint of any $I_i$ is at least $m^{-3/2}$ smaller than the left endpoint of $I_{i+1}$. 
    
    We can construct these intervals in an inductive fashion. First, let $\gamma \triangleq 1 - \sum_i \lambda_i = \Theta(1/\sqrt{m})$. Let $I_1 = [a_1,b_1]$ for $a_1 < b_1 < 0$ such that $\gamma(I_1) = \lambda_1$ and $\gamma((-\infty,a_1]) = \gamma/m$. Given $I_1,\ldots,I_i$ for $1\le i < (m-1)/2$, if $I_i = [a_i,b_i]$ is the right endpoint of $I_i$, then let $a_{i+1} > b_i$ be such that $\gamma([b_i,a_{i+1}) =  \gamma/m$, and define $I_{i+1} = [a_{i+1},b_{i+1}]$ for $b_{i+1} > a_{i+1}$ satisfying $\gamma([a_{i+1},b_{i+1}]) = \lambda_{i+1}$. By construction,
    \begin{equation}
        \gamma((-\infty,b_{(m-1)/2}]) = \sum^{(m-1)/2}_{i=1} \gamma(I_i) + \frac{m-1}{2}\cdot \frac{\gamma}{m} = \frac{1-\gamma}{2} + \frac{m-1}{2m}\cdot \frac{\gamma}{2} = \frac{1}{2} - \frac{\gamma}{2m},
    \end{equation}
    so by Gaussian anticoncentration and the fact that $\gamma = \Theta(1/\sqrt{m})$, we conclude that $b_{(m-1)/2} \le -\Omega(m^{-3/2})$. In the same way, we also conclude that because $\gamma([b_i,a_{i+1}]) = \frac{\gamma}{m}$, we must have $b_i - a_{i+1} \le - \Omega(m^{-3/2})$. Finally, for $i = (m-1)/2 + 1,\ldots,m-1$, we can define $I_i$ to be the reflection of $I_{m-i}$ about the origin. Note that by our bounds on $b_i - a_{i+1}$ for $1 \le i \le (m-1)/2$ and on $b_{(m-1)/2}$, all of the intervals are $\Omega(m^{-3/2})$-separated from each other as claimed. And by design, $\gamma(I_i) = \lambda_i$ for all $1\le i \le m - 1$.
    
    While the lemma is stated in terms of $\epsilon > 0$, let us first consider the following construction where $\epsilon = 0$. We can take the centers $c_1,\ldots,c_{m-1}$ in the lemma to be the centers of $I_1,\ldots,I_{m-1}$, and $w_1,\ldots,w_{m-1}$ to be half of the widths of $I_1,\ldots,I_{m-1}$, in which case $f \triangleq \sum^{m-1}_{i=1} T^{w_i,h_i,0}_{c_i}$ immediately satisfies Parts~\ref{boxes:separated} and \ref{boxes:sym} of the lemma. Then the pushforward of $\calN(0,1)$ under this choice of $f$ is the distribution which with probability $\gamma$ equals zero (when $z\sim\calN(0,1)$ lies outside of $I_1,\ldots,I_{m-1}$) and otherwise takes the value $h_i$ with probability $\lambda_i$. Parts~\ref{boxes:moment}, \ref{boxes:bounded}, and \ref{boxes:thin} then follow from Lemma~\ref{lem:dks2}. Finally, note that $a_1$ defined above is at most $O(\log m)$ in magnitude because $\gamma((-\infty,a_1]) = \gamma/m$ by Part~\ref{dks2:nottoosmall} of Lemma~\ref{lem:dks2}. This establishes Part~\ref{boxes:endpoints} of the lemma.
    
    Finally, note that by taking $\epsilon$ infinitesimally small (relative to $\nu$) but positive, the function $f$ defined in \eqref{eq:infinitebump_def} satisfies all of the parts of the lemma.
\end{proof}

\noindent Unfortunately the slopes in the piecewise linear function constructed in Lemma~\ref{lem:boxes} are arbitrarily large if $\nu$ is arbitrarily small. The issue still remains of how to get a continuous piecewise-linear function whose slopes are \emph{polynomially bounded} so that the corresponding ReLU network has polynomially bounded weights. As we illustrate in the next subsection however, the $\Omega(m^{-3/2})$ spacing between the ``bumps'' in the definition of $f$ in Lemma~\ref{lem:boxes} gives us sufficient room to carefully perturb $f$ to achieve this goal.

\subsubsection{Estimates for Bump Moments}

For convenience, define
\begin{equation}
    M^{w,h,\epsilon}_{c,k} \triangleq \E[g]*{T^{w,h,\epsilon}_c(g)^k}.
\end{equation}
We conclude this subsection by collecting some useful bounds for this quantity. First, we give an explicit expression for these moments:
\begin{lemma}\label{lem:moment_bump}
    For $c,w,h,\epsilon > 0$ satisfying $c - \epsilon - w \ge 0$, we have
    \begin{multline}
        M^{w,h,\epsilon}_{c,k} = h^k\gamma([c-w,c+w]) + \\ (k-1)!!\left(\frac{h}{\epsilon}\right)^k\sum^k_{i=0 \ \text{even}} \binom{k}{i} \Bigg[ (-c+\epsilon+w)^{k-i}\gamma([c-\epsilon-w,c-w])
        + (c+\epsilon+w)^{k-i}\gamma([c+w,c+\epsilon+w]) \Bigg] \\
        - \left(\frac{h}{\epsilon}\right)^k\sum^k_{i=0} \binom{k}{i} \Bigg[  (-c+\epsilon+w)^{k-i}\left(p_i(c-w)\gamma(c-w) - p_i(c - \epsilon-w)\gamma(c-\epsilon-w)\right) \\
        + (-1)^i (c+\epsilon+w)^{k-i}\left(p_i(c+\epsilon+w)\gamma(c+\epsilon+w) - p_i(c+w)\gamma(c+w)\right) \Bigg]
    \end{multline} for all even $k$.
\end{lemma}

\noindent We defer the proof of this to Appendix~\ref{defer:moment_bump}. The particular form for this expression is not too important; we will only use it to make clear that $M^{w,h,\epsilon}_{c,k}$ is continuously differentiable with respect to $\epsilon$ when $\epsilon > 0$, and to obtain the following expression for the derivative of the moments with respect to $h$:

\begin{lemma}\label{cor:hderiv}
    \begin{equation}
        \frac{\partial M^{w,h,\epsilon}_{x,k}}{\partial h} = \frac{k}{h} M^{w,h,\epsilon}_{x,k}
    \end{equation}
\end{lemma}

\begin{proof}
    This is immediate from Lemma~\ref{lem:moment_bump}.
\end{proof}

\noindent We will also need the following bound showing that $M^{w,h,\epsilon}_{c,k}$ does not change too much in a neighborhood of $(h,\epsilon)$:

\begin{lemma}\label{lem:Mstable}
    For any $w\ge 0$, $c,h,h'\in\R$, $\epsilon'\ge \epsilon\ge 0$, and even $k\in\mathbb{N}$,
    \begin{equation}
        |M^{w,h',\epsilon'}_{c,k} - M^{w,h,\epsilon}_{c,k}| \le h^k\left(|(h'/h)^k - 1| + \epsilon'-\epsilon\right).
    \end{equation}
    In particular, this implies that $\left|\frac{\partial M^{w,h,\epsilon}_{c,k}}{\partial\epsilon}\right| \le h^k$.
\end{lemma}

\begin{proof}
    Note that $0 \le M^{w,h,\epsilon}_{c,k} \le h^k\gamma([c-\epsilon-w,c+\epsilon+w])$, so the first part of the lemma follows by
    \begin{align}
        \MoveEqLeft |M^{w,h',\epsilon'}_{c,k} - M^{w,h,\epsilon}_{c,k}|  \\
        &\le |{h'}^k - h^k| \gamma([c-\epsilon'-w,c+\epsilon'+w]) + h^k\gamma([c-\epsilon'-w,c-\epsilon-w]\cup[c+\epsilon+w,c+\epsilon'+w]) \\
        &\le |{h'}^k - h^k| + h^k(\epsilon'-\epsilon) = h^k\left(|(h'/h)^k - 1| + \epsilon'-\epsilon\right),
    \end{align}
    where in the last step we used that $\gamma([a,a+\eta]) \le \eta/2$ for any $a\in\R$, $\eta \ge 0$. The second part of the lemma then follows by taking $h = h'$ and $\epsilon'\to\epsilon$.
\end{proof}


\subsection{ODE-Driven Perturbation}
\label{sec:ode}

Denote the parameters of the function constructed in Lemma~\ref{lem:boxes} by $\brc{(h_i(0),w_i,c_i)}_{1\le i < m}$. We will also define $\epsilon(0)$ to be some arbitrarily small positive quantity satisfying $\epsilon(0) \le \overline{\epsilon}$ for the parameter $\overline{\epsilon}$ from Lemma~\ref{lem:boxes}.

We will design an ordinary differential equation whose solution specifies a one-parameter family of functions
\begin{equation}
    f_t \triangleq \sum^{m-1}_{i=1} T^{w_i,h_i(t),\epsilon(t)}_{c_i} \label{eq:ftdef}
\end{equation}
that arise from gradually perturbing the function from Lemma~\ref{lem:boxes}. Roughly speaking, starting at $h_i(0)$ and $\epsilon(0)$ for all $1\le i < m$, perturbing the function along this one-parameter family will correspond to keeping the widths $w_i$ and centers $c_i$ of the bumps fixed, increasing the $\epsilon$ parameter of every bump at unit speed, and evolving the heights $h_i(t)$ in such a way that the moments of the pushforward of $\calN(0,1)$ under $f_t$ remain constant in $t$ for all $0 \le t \le T$. We illustrate this evolution in Figure~\ref{fig:evolve}. Here $T$ is some horizon which is at least inverse-polynomially large but smaller than $m^{-3/2}$ so that the ``edges'' $c_i \pm (\epsilon(t) + w_i)$ of the bumps don't collide with each other (this is where we make crucial use of Part~\ref{boxes:separated} of Lemma~\ref{lem:boxes}). At the end of this horizon, we want to show that the heights will not have changed too much, whereas the bumps now have $\epsilon$ parameter given by inverse-polynomially large $T$. This will imply that $f_T$ has polynomially bounded slopes as desired.

\begin{figure}[h]
    \centering
    \includegraphics[width=0.9\textwidth]{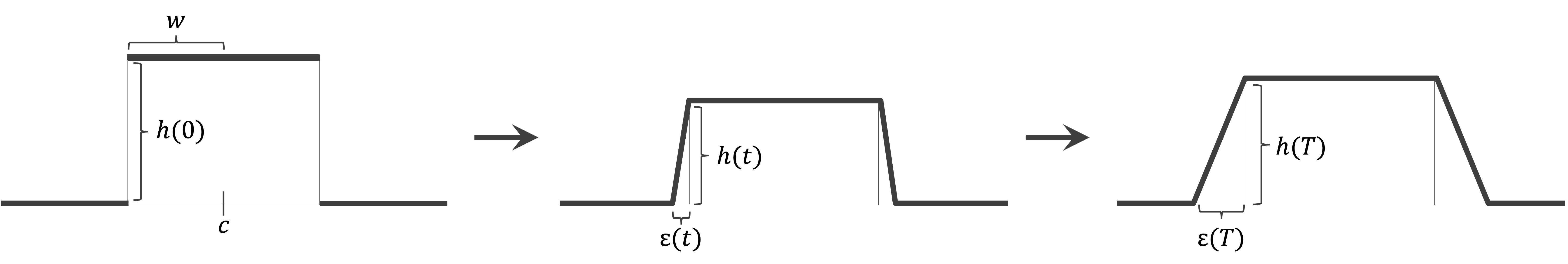}
    \caption{Evolution of one of the $m - 1$ bumps constituting $f_t$}
    \label{fig:evolve}
\end{figure}

As the odd moments of $f_0(\calN(0,1))$ vanish and the parameters $\brc{(w_i,h_i(0),c_i)}_{1\le i < m}$ satisfy the symmetry properties from Part~\ref{boxes:sym} of Lemma~\ref{lem:boxes}, it is easy to ensure that the odd moments of $f_t(\calN(0,1))$ also vanish: simply take $h_i(t) = -h_{m-i}(t)$ for all $1\le i \le m-1$. 

We thus focus on evolving $(h_1(t),\ldots,h_{(m-1)/2}(t))$. For convenience denote this by $\vec{h}(t)$. Define moment vector $\mu:\R^{(m-1)/2}\times \R\to\R^{m-1}$ by 
\begin{equation}
    \mu(\vec{h},\epsilon) \triangleq \brc*{\sum^{(m-1)/2}_{i=1} M^{w_i,h_i,\epsilon}_{c_i,2\ell}}_{1\le \ell \le (m-1)/2}.
\end{equation}
For any fixed $(\vec{h},\epsilon)$ in a small neighborhood of $(\vec{h}(0),\epsilon(0))$, we want to show there is a direction $v\in\R^{(m-1)/2}$ such that the directional derivative of $\mu$ in the direction $w\triangleq (v_1,\ldots,v_{(m-1)/2},1)$ is zero. Note that the constraint that 
\begin{equation}
    \nabla_w \mu(\vec{h},\epsilon) = \vec{0}
\end{equation}
specifies a system of linear constraints of the form
\begin{equation}
    \sum^{(m-1)/2}_{i=1} v_i \cdot \frac{\partial M^{w_i,h_i,\epsilon}_{c_i,2k}}{\partial h_i} = -\sum^{(m-1)/2}_{i=1} \frac{\partial M^{w_i,h_i,\epsilon}_{c_i,2k}}{\partial \epsilon} \ \ \forall \ 1\le \ell \le (m-1)/2. \label{eq:system}
\end{equation}
Recalling Lemma~\ref{cor:hderiv}, we can rewrite this as
\begin{equation}
    \sum^{(m-1)/2}_{i=1} v_i \cdot \frac{2k}{h_i} M^{w_i,h_i,\epsilon}_{c_i,2k} = -\sum^{(m-1)/2}_{i=1} \frac{\partial M^{w_i,h_i,\epsilon}_{c_i,2k}}{\partial \epsilon} \ \ \forall \ 1\le \ell \le (m-1)/2. \label{eq:system2}
\end{equation}
To express this more compactly, define $\vec{b}(\epsilon)\in\R^{(m-1)/2}$ and $Z(\vec{h},\epsilon)$ by
\begin{equation}
    \vec{b}(\epsilon)_{\ell} \triangleq -\sum^{(m-1)/2}_{i=1} \frac{\partial M^{w_i,h_i,\epsilon}_{c_i,2\ell}}{\partial \epsilon} \qquad \text{and} \qquad Z(\vec{h},\epsilon)_{i,\ell} \triangleq M^{w_i,h_i,\epsilon_i}_{c_i,2\ell}. \label{eq:bZdef}
\end{equation}
Also define the matrices $A(\vec{h}) \triangleq \diag(1/h_1,\ldots,1/h_{(m-1)/2})$ and $B \triangleq \diag(2,4,\ldots,m-1)$.
Then \eqref{eq:system} is equivalent to
\begin{equation}
    v^{\top} \cdot A(\vec{h})Z(\vec{h},\epsilon)B = \vec{b}(\epsilon)^{\top}.
\end{equation}
Provided $A(\vec{h})Z(\vec{h},\epsilon)B$ is invertible, the natural choice for $v$ would thus be the vector given by $v = B^{-1}Z(\vec{h},\epsilon)^{-\top}A(\vec{h})^{-1} \cdot \vec{b}(\epsilon)$. Therefore, defining
\begin{equation}
    w(t,\vec{h}) \triangleq \left(B^{-1}Z(\vec{h},\epsilon(0) + t)^{-\top} A(\vec{h})^{-1} \cdot \vec{b}(\epsilon(0) + t),1\right), \label{eq:wdef}
\end{equation}
we consider the following initial value problem
\begin{equation}
    \vec{h}'(t) = w(t,\vec{h}(t)) \qquad \text{and} \qquad \vec{h}(0) = (h_1(0),\ldots,h_{(m-1)/2}(0)). \label{eq:IVP}
\end{equation}
Note that if we had a solution $\vec{h}(t)$ to \eqref{eq:IVP} for $t\in[0,T]$ for some horizon $T$, then we would have
\begin{equation}
    \frac{\partial}{\partial t}\mu(\vec{h}(t),t) = \left(\frac{\partial}{\partial (\vec{h}(t),t)} \mu(\vec{h}(t),t)\right)\cdot \vec{h}'(t) = \nabla_{w(t)}\mu(\vec{h}(t),t) = \vec{0}, \label{eq:mainode}
\end{equation}
implying that the low-degree moments of $f_t$ defined in \eqref{eq:ftdef} are constant in $t$ as desired.


\subsection{Existence and Boundedness of \texorpdfstring{$\vec{h}(t)$}{h(t)}}

To carry out the strategy outlined in Section~\ref{sec:ode}, we must establish that
\begin{enumerate}[leftmargin=*]
    \item A solution to the initial value problem \eqref{eq:IVP} exists over a non-negligible horizon $T$.
    \item The entries of $\vec{h}(t)$ do not explode in $t$.
\end{enumerate}

For both of these, we need to show that the matrix $Z(\vec{h},\epsilon)$ is invertible or, more specifically, well-conditioned for $(\vec{h},\epsilon)$ in a neighborhood of $(\vec{h}(0),\epsilon(0))$. We first establish this at time $t = 0$ by relating $Z(\vec{h}(0),\epsilon(0))$ to a certain Vandermonde matrix and appealing to Fact~\ref{fact:vandermonde}:

\begin{lemma}\label{lem:timezerocond}
    $\sigma_{\min}(Z(\vec{h}(0),\epsilon(0))) \ge m^{-Cm}$ for an absolute constant $C > 0$.
\end{lemma}

\begin{proof}
    For convenience, in this proof we refer to $Z(\vec{h}(0),\epsilon(0))$ as $Z$. Note that $Z_{i,\ell} = \gamma([c_i-w_i,c_i+w_i])\cdot h_i(0)^{2\ell} + \xi_{i,\ell}$ for some $\xi_{i,\ell}$ which can be made arbitrarily small by taking $\epsilon(0)$ to be arbitrarily small. We can thus write $Z = \Lambda H + \Xi$ for $\Lambda = \diag(\lambda_1 h_1(0)^2,\ldots,\lambda_{(m-1)/2} h_{(m-1)/2}(0)^2)$, $H\in\R^{(m-1)/2\times (m-1)/2}$ given by $H_{i,\ell} = h_i(0)^{2\ell-2}$, and $\Xi$ a matrix consisting of arbitrarily small positive entries. So $\sigma_{\min}(Z) \ge (\min_i \lambda_i h_i(0)^2)\cdot \sigma_{\min}(H) - \xi \ge (e^{-cm}/m) \cdot \sigma_{\min}(H) - \xi$ for arbitrarily small $\xi > 0$, where in the last step we used Parts~\ref{boxes:bounded} and \ref{boxes:thin} of Lemma~\ref{lem:boxes}.
    
    Finally, note that $H$ is a Vandermonde matrix with nodes $h_1(0)^2,\ldots,h_{(m-1)/2}(0)^2$. As $\brc{h_i(0)}$ are $\Omega(1/\sqrt{m})$-separated and lie within $[\Omega(1/\sqrt{m}),O(\sqrt{m})]$, $\brc{h_i(0)^2}$ are $\Omega(1/m)$-separated. So by Fact~\ref{fact:vandermonde}, $\sigma_{\min}(H) \ge m^{-O(m)}$, concluding the proof.
\end{proof}

\noindent We can use Lemma~\ref{lem:timezerocond} to deduce that for $(\vec{h},\epsilon)$ in a neighborhood of $(\vec{h}(0),\epsilon(0))$, $Z(\vec{h},\epsilon)$ is also well-conditioned:

\begin{lemma}\label{lem:latercond}
    Let $C > 0$ be the absolute constant from Lemma~\ref{lem:timezerocond}. For any $(\vec{h},\epsilon)$ satisfying $\norm{\vec{h} - \vec{h}(0)}_{\infty} \le m^{-C'm}$ and $0 \le  \epsilon - \epsilon(0) \le m^{-C'm}$ for sufficiently large absolute constant $C'>0$, $\sigma_{\min}(Z(\vec{h},\epsilon)) \ge m^{-Cm}/2$.
\end{lemma}

\begin{proof}
    For convenience, in this proof we refer to $Z(\vec{h},\epsilon)$ and $Z(\vec{h}(0),\epsilon(0))$ by $Z'$ and $Z$ respectively. By Lemma~\ref{lem:Mstable}, each entry of $Z'$ differs from the corresponding entry of $Z$ by at most
    \begin{equation}
        (\max_i h_i^{m-1})\cdot \left(\Big|(1+m^{-C'm}/(\min_i h_i))^{m-1}-1\Big| + m^{-C'm} - \epsilon(0)\right).
    \end{equation}
    As $\max_i h_i \le O(\sqrt{m})$ and $\min_i h_i \ge \Omega(1/\sqrt{m})$ by Part~\ref{boxes:bounded} of Lemma~\ref{lem:boxes}, and $\epsilon(0)$ can is an arbitrarily small positive quantity, the above is at most $m^{-C''m}$ for some absolute constant $C'' > 0$ which is increasing in $C'$. So $\norm{Z - Z'}_\op \le \norm{Z - Z'}_F \le (m-1)/2 \cdot m^{-C''m} \ll m^{-Cm}$ provided we take $C'$ sufficiently large.
\end{proof}

\noindent To establish Property 1, we must first verify that $w(t,\vec{h})$ is continuous:

\begin{lemma}\label{lem:cts}
    Let $C' > 0$ be the absolute constant from Lemma~\ref{lem:latercond}. Then the function $w(t,\vec{h})$ defined in \eqref{eq:wdef} is continuous with respect to both $t$ and $\vec{h}$ for $t\le m^{-C'm}$ and $\norm{\vec{h} - \vec{h}(0)}_{\infty} \le m^{-C'm}$.
\end{lemma}

\begin{proof}
    By our expression for $M^{w,h,\epsilon}_{c,k}$ in Lemma~\ref{lem:moment_bump} and the definition of $\vec{b}(\epsilon)$ in \eqref{eq:bZdef}, $\vec{b}(\epsilon(0) +t)$ is clearly continuous in $t$ whenever $t \ge 0$ (because $\epsilon(0) > 0$). Similarly, $A(\vec{h}) Z(\vec{h},\epsilon(0) + t) B$ is clearly continuous with respect to $\vec{h}$ and $t$ whenever $t \ge 0$ and $h_i\neq 0$ for all $i$. By Lemma~\ref{lem:latercond}, if $t\le m^{-C'm}$ and $\norm{\vec{h} - \vec{h}(0)}_{\infty} \le m^{-C'm}$ (which additionally implies that $h_i \neq 0$ for all $i$, by Part~\ref{boxes:bounded} of Lemma~\ref{lem:boxes}), then $A(\vec{h}) Z(\vec{h},\epsilon(0) + t) B$ is invertible. We conclude that for such $t,\vec{h}$, $w$ is continuous.
\end{proof}

\noindent Lastly, we need to show that for any $(t,\vec{h})$ satisfying the hypotheses of Lemma~\ref{lem:cts}, $\norm{w(t,\vec{h})}_{\infty}$ is not too large, which will imply Property 1 above by Theorem~\ref{thm:peano} and Property 2 above by noting that $\norm{\vec{h}'(t)}_{\infty} = \norm{w(t,\vec{h}(t))}_{\infty}$:

\begin{lemma}\label{lem:wbound}
    Let $C' > 0$ be the absolute constant from Lemma~\ref{lem:latercond}. If $t\le m^{-C'm}$ and $\norm{\vec{h} - \vec{h}(0)}_{\infty} \le m^{-C'm}$, then $\norm{w(t,\vec{h})}_\infty \le m^{C'' m}$ for some absolute constant $C'' > 0$.
\end{lemma}

\begin{proof}
    By the second part of Lemma~\ref{lem:Mstable}, every entry of $\vec{b}(\epsilon(0) + t)$ is at most
    \begin{equation}
        \frac{m-1}{2}\cdot (\max_i h_i)^{m-1} \le \frac{m-1}{2}\cdot \left(O(\sqrt{m}) + m^{-C'm}\right)^{m-1} \le m^{O(m)},
    \end{equation}
    where in the penultimate step we used Part~\ref{boxes:bounded} of Lemma~\ref{lem:boxes} and our hypothesis on $\vec{h}$. Note that $\sigma_{\min}(Z(\vec{h},\epsilon(0)+t))\ge m^{-Cm}/2$ by Lemma~\ref{lem:latercond}, $\sigma_{\min}(B) \ge 2$, and $\sigma_{\min}(A(\vec{h})) \ge \min_i 1/h_i \ge \Omega(\sqrt{m})$ by Part~\ref{boxes:bounded} of Lemma~\ref{lem:boxes} and our hypothesis on $\vec{h}$. We conclude that $B^{-1}Z(\vec{h},\epsilon(0)+t)^{-\top} A(\vec{h})^{-1} \cdot \vec{b}(\epsilon(0)+t)$ has $L_\infty$ norm at most $m^{C''m}$ for some absolute constant $C'' > 0$, so $\norm{w(t,\vec{h})}_\infty \le m^{C'' m}$ as claimed.
\end{proof}

\noindent We are now ready to put all of these ingredients together to prove the key lemma which will allow us to establish Theorem~\ref{thm:moment}:

\begin{lemma}\label{lem:useode}
    Fix any odd $m\in\mathbb{N}$ and any $0 < \nu < 1$. There is a one-hidden-layer ReLU network $f: \R\to\R$ of size $O(m)$ with weights at most $m^{O(m)}$ for which the pushforward $D = f(\calN(0,1))$ satisfies $|\E[x\sim D]{x^k} - \E[g\sim\calN(0,1)]{g^k}| < \nu$ for all $k = 1,\ldots,m$.
\end{lemma}

\begin{proof}
    Let $B$ be the parallelepiped consisting of $(t,\vec{h})$ for which $0\le t \le m^{-C'm}$ and $\norm{\vec{h} - \vec{h}(0)}_{\infty} \le m^{-C'm}$. By Lemma~\ref{lem:cts}, $w$ is continuous over $B$. By Lemma~\ref{lem:wbound}, $\norm{w(t,\vec{h})}_{\infty} \le m^{C''m}$.
    
    By Theorem~\ref{thm:peano} we conclude that the initial value problem in \eqref{eq:IVP} has a solution $\vec{h}(t)$ over $t\in[0,T]$ for 
    \begin{equation}
        T = m^{-(C'+C'')m}. \label{eq:Tdef}
    \end{equation} Furthermore, because $\norm{\vec{h}'(t)}_\infty = \norm{w(t,\vec{h}(t))}_\infty \le m^{C''m}$, we conclude that $\frac{1}{T}\norm{\vec{h}(T)}_\infty \le m^{C''m}$. The slopes of the bumps $T^{w_i,h_i(T),\epsilon(0)+T}_{c_i}$ are therefore bounded by $m^{C''m}$.
    
    For $(m-1)/2 < i \le m - 1$, define $h_i(T) = -h_{m-i}(T)$ and consider the one-parameter family of functions $f_t \triangleq \sum^{m-1}_{i=1} T^{w_i,h_i(t),\epsilon(0)+t}_{c_i}$. Because $c_i = c_{m-i}$ and $w_i = w_{m-i}$ for all $1\le i < m$ by Part~\ref{boxes:sym} of Lemma~\ref{lem:boxes}, we conclude that the odd moments of $f_T(\calN(0,1))$ all vanish. As for the even moments, because $\frac{\partial}{\partial t}\mu(\vec{h}(t),t) = \vec{0}$ by \eqref{eq:mainode}, we conclude that the even moments of $f_T(\calN(0,1))$ up to degree $m - 1$ agree with those of $f_0(\calN(0,1))$. So because $D_0 = f_0(\calN(0,1))$ satisfies $|\E[x\sim D]{x^k} - \E[g\sim\calN(0,1)]{g^k}| < \nu$ for all $k = 1,\ldots,m$ by Part~\ref{boxes:moment} of Lemma~\ref{lem:boxes}, we conclude that the same holds for $f_T$.
    
    Finally, as the endpoints of the intervals supporting the bumps are bounded by $O(\log m)$, we conclude from Fact~\ref{fact:implementbox} that $f_T \triangleq \sum^{m-1}_{i=1} T^{w_i,h_i(T),\epsilon(0)+T}_{c_i}$ can be implemented by a one-hidden-layer ReLU network of size $O(m)$ with weights at most $m^{O(m)}$. 
\end{proof}


\subsection{Proof of Theorem~\ref{thm:moment}}

So far in Lemma~\ref{lem:useode} we have constructed a pushforward distribution computed by a one-hidden-layer ReLU network with $m^{O(m)}$-bounded weights satisfying the first part of Theorem~\ref{thm:moment}. In this subsection we slightly modify this construction to additionally satisfy the remaining two parts of Theorem~\ref{thm:moment}. The proofs here are standard (see e.g. \cite{diakonikolas2017statistical}) and we defer them to the appendix.

First, to ensure that the chi-squared divergence between the pushforward and $\calN(0,1)$ is not too large, we simply convolve (a scaling of) the pushforward by a suitable Gaussian. The following lemma gives an estimate for the resulting chi-squared divergence:

\begin{lemma}\label{lem:chisq}
    Let $A$ be any distribution supported on an interval $[-R,R]$. Then for any $0 < \sigma \le 1/2$, $\chi^2(A\star\calN(0,\sigma^2),\calN(0,1)) \le e^{O(R^2)}/\sigma$.
\end{lemma}

\noindent We defer the proof of this to Appendix~\ref{defer:chisq}. Next, we verify that appropriately scaling and convolving by a Gaussian doesn't alter the moments of a distribution whose low-degree moments match those of $\calN(0,1)$:

\begin{lemma}\label{lem:momentconvolve}
    Let $D$ be any symmetric distribution for which $|\E[x\sim D]{x^k} - \E[g]{g^k}| < \nu$ for all $1\le k < m$. For any $c\in\R$ let $c\cdot D$ denote the distribution obtained by rescaling $D$ by a factor of $c$. Then $D' \triangleq \sqrt{1-\sigma^2}\cdot D \star \calN(0,\sigma^2)$ satisfies $|\E[x\sim D']{x^k} - \E[g]{g^k}| < \nu$ for all $1 \le k < m$.
\end{lemma}

\noindent We defer the proof of this to Appendix~\ref{defer:momentconvolve}. Finally, we show that scaling the pushforward from Lemma~\ref{lem:useode} and convolving by a Gaussian yields a distribution satisfying the third part of Theorem~\ref{thm:moment}:

\begin{lemma}\label{lem:projtv}
        Let $D = f(\calN(0,1))$ be from Lemma~\ref{lem:useode}, and define $D' \triangleq \sqrt{1-\sigma^2}\cdot D \star \calN(0,\sigma^2)$. Then for any $v,v'\in\S^{d-1}$ satisfying $|\iprod{v,v'}| \le 1/2$, $d_{\mathrm{TV}}(P^{D'}_v,P^{D'}_{v'}) \ge 1 - 2\sigma\log(1/\sigma) - m^{-\Omega(m)}$.
\end{lemma}

\noindent We defer the proof of this to Appendix~\ref{defer:projtv}. We are now ready to complete the proof of Theorem~\ref{thm:moment}.

\begin{proof}[Proof of Theorem~\ref{thm:moment}]
    Let $f$ be the function constructed in Lemma~\ref{lem:useode} and define $f^*:\R^2\to\R$ by $f^*(z_1,z_2) = \sqrt{1-\sigma^2} f + \sigma z_2$. Note that $f^*(\calN(0,\Id))$ is exactly $\sqrt{1-\sigma^2} f(\calN(0,1)) \star \calN(0,\sigma^2)$, so the three parts of Theorem~\ref{thm:moment} follow immediately from Lemmas~\ref{lem:momentconvolve}, \ref{lem:chisq}, and \ref{lem:projtv} respectively.
\end{proof}

\paragraph{Acknowledgments.} We thank Adam R. Klivans, Raghu Meka, and Anru R. Zhang for many illuminating discussions about learning generative models.


\bibliographystyle{alpha}
\bibliography{biblio}

\appendix

\section{Deferred Proofs}

\subsection{Proof of Lemma~\ref{lem:moment_bump}}
\label{defer:moment_bump}

\begin{proof}
    By Corollary~\ref{cor:shift_moment}, the contribution from the interval $g\in[c-\epsilon-w,c-w]$ to $\E[g]*{\left(T^{w,h,\epsilon}_c(z)\right)^k}$ is given by
    \begin{multline}
        \sum^k_{i=0 \ \text{even}} \binom{k}{i}\left(\frac{h}{\epsilon}\right)^i \left(\frac{h}{\epsilon}(-c+\epsilon+w)\right)^{k-i}(k-1)!! \gamma([c-\epsilon-w,c-w]) - \\
        \sum^k_{i=0} \binom{k}{i}\left(\frac{h}{\epsilon}\right)^i \left(\frac{h}{\epsilon}(-c+\epsilon+w)\right)^{k-i}\left(p_i(c-w)\gamma(c-w) - p_i(c - \epsilon-w)\gamma(c-\epsilon-w)\right).
    \end{multline}
    Similarly, the contribution from the interval $g\in[c+w,c+\epsilon+w]$ is given by
    \begin{multline}
        \sum^k_{i=0 \ \text{even}} \binom{k}{i}\left(\frac{-h}{\epsilon}\right)^i \left(\frac{h}{\epsilon}(c+\epsilon+w)\right)^{k-i}(k-1)!! \gamma([c+w,c+\epsilon+w]) - \\
        \sum^k_{i=0} \binom{k}{i}\left(\frac{-h}{\epsilon}\right)^i \left(\frac{h}{\epsilon}(c+\epsilon+w)\right)^{k-i}\left(p_i(c+\epsilon+w)\gamma(c+\epsilon+w) - p_i(c+w)\gamma(c+w)\right).
    \end{multline}
    Finally, the contribution from the interval $g\in[c-w,c+w]$ is given by $h^k\cdot \gamma([c-w,c+w])$.
\end{proof}

\subsection{Proof of Lemma~\ref{lem:chisq}}
\label{defer:chisq}

\begin{proof}
    Let $\wt{A} \triangleq A\star\calN(0,1)$. By definition, $\wt{A}(x) = \int^\infty_{-\infty} A(s) \gamma_{\sigma^2}(x - s) \, \mathrm{d}s$. So
    \begin{equation}
        1 + \chi^2(\wt{A},\calN(0,1)) = \int^\infty_{-\infty} \frac{1}{\gamma(x)} \left(\int^\infty_{-\infty}\int^\infty_{-\infty} A(s)A(t)\gamma_{\sigma^2}(x-s)\gamma_{\sigma^2}(x-t) \, \mathrm{d}s\mathrm{d}t\right) \mathrm{d}x. \label{eq:onepluschi}
    \end{equation}
    Note that for any $s,t\in\R$,
    \begin{align}
        \frac{\gamma_{\sigma^2}(x-s)\gamma_{\sigma^2}(x-t)}{\gamma(x)} &= \frac{1}{\sigma^2\sqrt{2\pi}} \exp\left(\frac{-(x-s)^2 - (x-t)^2}{2\sigma^2} + x^2/2\right) \\
        &= \frac{1}{\sigma^2\sqrt{2\pi}} \exp\left(-\frac{2-\sigma^2}{2\sigma^2}\left(x-\frac{s+t}{2-\sigma^2}\right)^2 + \frac{2st-(s^2+t^2)(1-\sigma^2)}{2\sigma^2(2-\sigma^2)}\right),
    \end{align}
    so
    \begin{equation}
        \int^\infty_{-\infty} \frac{\gamma_{\sigma^2}(x-s)\gamma_{\sigma^2}(x-t)}{\gamma(x)} \,\mathrm{d}x = \frac{e^\frac{2st-(s^2+t^2)(1-\sigma^2)}{2\sigma^2(2-\sigma^2)}}{\sigma\sqrt{2 - \sigma^2}}.
    \end{equation}
    Eq.~\eqref{eq:onepluschi} thus becomes
    \begin{equation}
        1 + \chi^2(\wt{A},\calN(0,1)) = \int^\infty_{-\infty}\int^\infty_{-\infty} A(s)A(t)\cdot \frac{e^\frac{2st-(s^2+t^2)(1-\sigma^2)}{2\sigma^2(2-\sigma^2)}}{\sigma\sqrt{2 - \sigma^2}} \,\mathrm{d}s \mathrm{d}t \label{eq:intst}
    \end{equation}
    As $A$ is supported on $[-R,R]$,
    \begin{equation}
        \frac{e^\frac{2st-(s^2+t^2)(1-\sigma^2)}{2\sigma^2(2-\sigma^2)}}{\sigma\sqrt{2 - \sigma^2}} \le e^{O(R^2/(1-\sigma^2))} \le e^{O(R^2)}.
    \end{equation}
    Substituting this into \eqref{eq:intst}, we find that
    \begin{align}
        1 + \chi^2(\wt{A},\calN(0,1)) &\le \frac{e^{O(k)}}{\sigma\sqrt{2-\sigma^2}} \int\int A(s)A(t) \, \mathrm{d}s\mathrm{d}t = \frac{e^{O(R^2)}}{\sigma\sqrt{2-\sigma^2}} \le e^{O(R^2)}/\sigma. \qedhere
    \end{align}
\end{proof}

\subsection{Proof of Lemma~\ref{lem:momentconvolve}}
\label{defer:momentconvolve}

\begin{proof}
    As the convolution is still a symmetric distribution, the odd moments clearly vanish. For any even $k < m$,
    \begin{align}
        \E[x\sim D']{x^k} &= \E[z\sim D,g \sim\calN(0,1)]*{(\sqrt{1-\sigma^2} z+\sigma g)^k} = \sum^{k/2}_{\ell=0} \binom{k}{2\ell} \sigma^{2\ell} (1 - \sigma^2)^{k/2-\ell} \E[z]{z^{2\ell}}\E[g]{g^{k-2\ell}} \\
        &\le \E[g,g'\sim\calN(0,1)]{(\sqrt{1-\sigma^2} g' + \sigma g)^k} + \nu \sum^{k/2}_{\ell=0} \binom{k}{2\ell} \sigma^{2\ell} (1 - \sigma^2)^{k/2-\ell} = \E[g\sim\calN(0,1)]{g^k} + \nu,
    \end{align}
    where in the third step we used that $\sqrt{1-\sigma^2}\cdot D$ matches the moments of $\calN(0,1 - \sigma^2)$ up to degree $m$ to error $\nu$, and in the last step we used that $\sqrt{1-\sigma^2} g' + \sigma g$ is distributed as a draw from $\calN(0,1)$. We can show in an identical fashion that $\E[x\sim D']{x^k} \ge \E[g]{g^k} - \nu$, completing the proof.
\end{proof}

\subsection{Proof of Lemma~\ref{lem:projtv}}
\label{defer:projtv}

\begin{proof}
    By Fact~\ref{fact:tvequiv}, it suffices to upper bound $\int_{\R^d} \min(P_v(z), P_{v'}(z)) \, \mathrm{d}z$. Let $H$ denote the plane spanned by $v,v'$. As the component in $H$ of a sample from either $P_v$ or $P_{v'}$ is independent from the component in $H^{\perp}$, and the latter is distributed as $\calN(0,\Pi_{H^{\perp}})$, it suffices to bound $\int_H \min(P_v(z), P_{v'}(z)) \, \mathrm{d}z$. Let $x,y$ be orthogonal coordinates for $H$ with $v$ in the direction of the $x$-axis, and let $x',y'$ be orthogonal coordinates for $H$ with $v'$ in the direction of the $x'$-axis. Let $\theta$ be the angle between $v,v'$. Then
    \begin{align}
        \int_H \min(P_v(z),P_{v'}(z))\, \mathrm{d}z &= \int^\infty_{-\infty} \int^\infty_{-\infty} \min(D'(x)\gamma(y), D'(x')\gamma(y')) \,\mathrm{d}x\mathrm{d}y \\
        &= \int^\infty_{-\infty}\int^\infty_{-\infty} \min(D'(x)\gamma(y), D'(x')\gamma(y')) \csc\theta \, \mathrm{d}x\mathrm{d}x'. \label{eq:minint} \\
    \end{align}
    For $1 \le i \le m -1$, let $D_i$ denote the distribution of $\sqrt{1-\sigma^2} \cdot T^{w_i,\vec{h}_i(T),T}_{x_i}(g)$ for $g\sim\calN(0,1)$ conditioned on $g\in[x_i - w_i - T, x_i + w_i + T]$. Also let $D_m$ denote the distribution which is a point mass at zero. Let $D'_i \triangleq D_i\star\calN(0,1)$. Note that there is a distribution $p$ over $[m]$ for which $D' = \E[i\sim p]{D'_i}$. Then we can upper bound \eqref{eq:minint} by
    \begin{equation}
        \max_{i,j\in[m]} \int^\infty_{-\infty}\int^\infty_{-\infty} \min(D'(x)_i\gamma(y), D'_j(x')\gamma(y')) \csc\theta \, \mathrm{d}x\mathrm{d}x' \label{eq:ubound}
    \end{equation}
    Note that for $1\le i \le m - 1$, $\Pr[x\sim D'_i]{|x - x_i| > a} \le \sigma + \xi_i$ for $a\triangleq 2\sigma\sqrt{\log(1/\sigma)}$ and $\xi_i \triangleq \frac{\gamma([x_i - w_i - T,x_i - w_i]\cup [x_i + w_i, x_i+ w_i + T])}{\gamma([x_i - w_i - T, x_i + w_i + T])}$. And for $i = m$, $\Pr[x\sim D'_m]{|x - x_i| > a} \le \sigma$ for $x_m \triangleq 0$. So we get an upper bound on \eqref{eq:ubound} of
    \begin{equation}
        \sigma + \max_{i\in[m-1]} \xi_i + \max_{i,j\in[m]} \int^{x_i + a}_{x_i - a} \int^{x_j + a}_{x_j - a} \min(D'_i(x)\gamma(y), D'_j(x')\gamma(y')) \csc\theta \, \mathrm{d}x\mathrm{d}x'
    \end{equation}
    As $D'_i$ is a convolution of $D_i$ with $\calN(0,\sigma^2)$, $D'_i(x) \le \frac{1}{\sigma\sqrt{2\pi}}$ for all $x\in\R$. And $\gamma(y) \le 1/\sqrt{2\pi}$, so the above display is at most
    \begin{equation}
        \sigma + \max_{i\in[m-1]} \xi_i + a^2\csc\theta / (2\pi\sigma) = \sigma + \max_{i\in[m-1]} \xi_i + \sigma\csc\theta\log(1/\sigma)/\pi.
    \end{equation}
    Note that if $|\iprod{v,v'}| \le 1/2$, then $\csc\theta / \pi\le 2/(\sqrt{3}\pi) \le 1$.
    Finally, to bound $\xi_i$, first note that for any $i\in[m-1]$, $\gamma([x_i - w_i - T, x_i + w_i + T]) = \lambda_i$, and recall that $\lambda_i \ge e^{-cm}$ for some absolute constant $c > 0$. On the other hand, $\gamma([x_i - w_i - T, x_i - w_i]) \le T/\sqrt{2\pi} = m^{-(C'+C'')m}/\sqrt{2\pi}$ by Gaussian anti-concentration and our choice of $T = m^{-(C'+C'')m}$ in the proof of Lemma~\ref{lem:useode}. So
    by taking the constants $C', C''$ to be larger than $c$, we conclude that $\xi_i \le m^{-\Omega(m)}$ for all $i\in[m-1]$.
\end{proof}

\section{Hardness for Estimation in Wasserstein}
\label{sec:wasserstein}

We now show an analogous version of Theorem~\ref{thm:mainlbd} under the Wasserstein-1 metric rather than total variation distance. We begin by observing that the ODE-based evolution does not move the pushforward at time zero, i.e. the distribution constructed in Lemma~\ref{lem:boxes}, too far away in Wasserstein distance over a time horizon of $T$:

\begin{lemma}\label{lem:wass_zero_vs_T}
    Let $f_0, f_T: \R\to\R$ denote the functions from Lemmas~\ref{lem:boxes} and \ref{lem:useode} respectively. Define $D_0\triangleq f_0(\calN(0,1))$ and $D_T \triangleq f_T(\calN(0,1))$. Then $W_1(D,D') \le m^{-\Omega(m)}$.
\end{lemma}

\begin{proof}
    Recall that $w_1,\ldots,w_{m-1}$ denote the widths of the bumps in $f_0,f_T$, $x_1,\ldots,x_{m-1}$ denote the centers, and the heights and $\epsilon$ parameters for the bumps in $f_0,f_T$ are given by $\brc{h_i(0)}_i,\epsilon(0)$ and $\brc{h_i(T)}_i,\epsilon(0)+T$ respectively, for $T = m^{-(C'+C'')m}$ and $\epsilon(0)$ an arbitrarily small positive quantity. Also recall from the proof of Lemma~\ref{lem:useode} that $|h_i(0) - h_i(T)| \le m^{-C'm}$ for all $i$.
    
    Now consider any $g\in\R$. If $g\in[x_i - w_i, x_i + w_i]$ for some $i$, then $|f_0(g) - f_T(g)| = |h_i(0) - h_i(T)| \le m^{-C'm}$. Furthermore, \begin{equation}
        \Pr[g\sim\calN(0,1)]{g\in[x_i - \epsilon_i(0) - T - w_i, x_i - w_i] \cup [x_i + w_i, x_i + \epsilon_i(0) + T + w_i] \ \text{for some} \ i} \le O(mT)
    \end{equation}
    Finally, for all $g$ that do not lie in any of the aforementioned intervals, i.e. that do not lie in the support of any bump from $f_0$ or $f_T$, note that $f_0(g) = f_T(g) = 0$ by construction. We conclude that for any 1-Lipschitz function $h:\R\to\R$,
    \begin{align}
        \abs*{\E[z\sim D]{h(z)} - \E[z\sim D']{h(z)}} &= \abs*{\E[g]{h(f_0(g)) - h(f_T(g))}} \\
        &\le \E[g]{|f_0(g) - f_T(g)|} \le m^{-C'm} + O(mT) \le m^{-\Omega(m)}
    \end{align} as claimed.
\end{proof}

\noindent We can now show the analogue of Lemma~\ref{lem:projtv} for Wasserstein distance:

\begin{lemma}\label{lem:projtv_wass}
    Let $D = f(\calN(0,1))$ be from Lemma~\ref{lem:useode}, and define $D' \triangleq \sqrt{1 - \sigma^2}\cdot D\star\calN(0,\sigma^2)$ for $\sigma \ll 1/\sqrt{m}$. Then for any $v,v'\in\S^{d-1}$ satisfying $|\iprod{v,v'}| \le 1/2$, $W_1(P^{D'}_v, P^{D'}_{v'}) \ge \Omega(1/\sqrt{m})$.
\end{lemma}

\begin{proof}
    Let $f_0$ be the function from Lemma~\ref{lem:boxes}, and let $A$ denote $f_0(\calN(0,1))$. We begin by lower bounding $W_1(P^{A}_v, P^{A}_{v'})$, which we will do by showing that with $\Omega(1)$ probability, a sample $x$ from $P^{A}_{v''}$ will be distance $\Omega(1/\sqrt{m})$ from the support of $P^{A}_v$. As the distance from a point $x$ to the affine hyperplane $\Lambda_h \triangleq \brc{z: \iprod{z,v} = h}$ is $|\iprod{v,x} - h|$, if $x$ is of the form $h' v' + v^{\perp}$ for some $h'\in\R$, then $x$ is at distance
    \begin{equation}
        \abs*{h' \iprod{v',v} + \iprod{v^{\perp},v} - h}
    \end{equation} from the hyperplane. Note that $P^{A}_v$ is supported on the hyperplanes $\Lambda_{h_1},\ldots,\Lambda_{h_{m-1}}$ for $h_1,\ldots,h_{m-1}$ from Lemma~\ref{lem:boxes}. And for $x\sim P^{A}_{v'}$, $h'$ takes on the value $h_i$ with probability $\lambda_i$ (where $\brc{\lambda_i}$ are also from Lemma~\ref{lem:boxes}), while $v^{\perp}$ is an independent draw from $\calN(0,\Id - v'{v'}^\perp)$. We conclude that $h'\iprod{v',v} + \iprod{v^\perp,v}$ is distributed as $\calN(h'\iprod{v',v}, 1 - \iprod{v',v}^2)$. Therefore, the event that $x$ is at distance $\Omega(1/\sqrt{m})$ from the support of $P^{A}_v$ is equivalent to the event that a sample from $\calN(h'\iprod{v',v}, 1 - \iprod{v',v}^2)$ is $\Omega(1/\sqrt{m})$-far from every $h_1,\ldots,h_{m-1}$. But note that because $h_1,\ldots,h_{m-1}$ are $\Omega(1/\sqrt{m})$-separated, there is an absolute constant $c > 0$ such that the union of the balls of radius $c/\sqrt{m}$ around $h_1,\ldots,h_{m-1}$ cover at most a constant fraction of the interval $[h'\iprod{v',v}-1, h'\iprod{v',v}+1]$. Because $1 - \iprod{v',v}^2\ge 3/4$, a constant fraction of the mass of $\calN(h'\iprod{v',v}, 1 - \iprod{v',v}^2)$ is located in this interval, concluding the proof that $W_1(P^{A}_v, P^{A}_{v'}) \ge \Omega(1/\sqrt{m})$.
    
    By Lemma~\ref{lem:wass_zero_vs_T} and the fact that scaling by $\sqrt{1 - \sigma^2}$ and convolving by $\calN(0,\sigma^2)$ incurs $O(\sigma) = o(1/\sqrt{m})$ in Wasserstein, we conclude that $W_1(P^{D'}_v, P^{A}_{v}) = W_1(D',A) = o(1/\sqrt{m})$ and similarly for $W_1(P^{D'}_{v'}, P^{A}_{v'})$. So by triangle inequality for Wasserstein, $W_1(P^{D'}_v, P^{D'}_{v'} = \Omega(1/\sqrt{m})$ as claimed.
\end{proof}

\noindent We conclude that in Theorem~\ref{thm:moment}, the distribution $D$ also satisfies the Wasserstein analogue of Part 3, i.e. $W_1(P^D_v, P^D_{v'}) \ge \Omega(1/\sqrt{m})$ for any $v,v'\in\S^{d-1}$ satisfying $|\iprod{v,v'}| \ge 1/2$. We can now prove an analogue of Theorem~\ref{thm:mainlbd} for Wasserstein:

\begin{theorem}\label{thm:mainlbd_wass}
    Let $d\in\mathbb{N}$ be sufficiently large. Any SQ algorithm which, given SQ access to an arbitrary one-hidden-layer ReLU network pushforward of $\calN(0,\Id_d)$ of size $O(\log d / \log\log d)$ with $\poly(d)$-bounded weights, outputs a distribution which is $O(\sqrt{\log\log d / \log d})$-close in $\dtv(\cdot)$ must make at least $d^{\Omega(\log d / \log\log d)}$ queries to either $\STAT(\tau)$ or $\VSTAT(1/\tau^2)$ for $\tau = d^{-\Omega(\log d / \log\log d)}$.
\end{theorem}

\begin{proof}
    By Theorem~\ref{thm:moment} applied with sufficiently large odd $m$ and sufficiently small $\sigma$, together with the above consequence of Lemma~\ref{lem:projtv_wass}, there exists a distribution $D = f^*(\calN(0,\Id_2))$ over $\R$ for $f^*:\R^2\to\R$ of size $O(m)$ with $m^{O(m)}$-bounded weights satisfying the hypotheses of Lemma~\ref{lem:generic} for $\epsilon = O(1/\sqrt{m})$, and $\chi^2(D,\calN(0,1)) \le \exp(O(m))$ (note that while Lemma~\ref{lem:generic} is stated for $\dtv(\cdot)$, it is also true with $\dtv(\cdot)$ replaced with Wasserstein-1). As long as $m \le d^{O(C)}$, we conclude that an SQ algorithm for learning any distribution from $\brc{P^D_v}_{v\in\S^{d-1}}$ to Wasserstein-1 distance $O(1/\sqrt{m})$ must make at least $d^{m+1}$ queries to $\STAT(\tau)$ or $\VSTAT(1/\tau^2)$ for $\tau \triangleq O(d)^{-(m+1)(1/4-C/2)}\cdot \exp(O(m))$. By taking $m = \Theta(\log d/ \log\log d)$, we ensure that $m^{O(m)} \le \poly(d)$. We're done by taking $C$ in Lemma~\ref{lem:generic} to be $C = 1/4$.
\end{proof}

\section{Hardness From Supervised Learning}
\label{sec:hardness_supervised}

In this section we make rigorous the claim from the introduction that lower bounds for PAC learning neural networks from Gaussian labeled examples imply lower bounds for learning neural network pushforwards. Formally, consider the following distinguishing problem:

\begin{definition}[Distinguishing labeled examples from Gaussian]\label{def:distinguish}
    For $d\in\mathbb{N}$, let $\calC_d$ be some class of functions from $\R^d$ to $\R$. The learner is given $\poly(d)$ many samples $(x_1,y_1),\ldots,(x_N,y_N)$ where $x_1,\ldots,x_N$ are independent draws from $\calN(0,\Id_d)$ such that one of the following is true: 1) there is some $h\in\calC$ for which $y_i = h(x_i)$ for all $i\in[N]$, or 2) every $y_i$ is an independent sample from $\calN(0,1)$. We say that an algorithm distinguishes between these two situations with constant advantage if the probability it outputs $\mathsf{YES}$ (resp. $\mathsf{NO}$) under the former (resp. latter) is at least $2/3$, where the probability is with respect to the randomness of the samples and internal randomness of the algorithm.
\end{definition}

\noindent Here we make the simple observation that an oracle for distinguishing any given family of non-Gaussian pushforwards from $\calN(0,\Id)$ (an easier task than actually learning pushforwards) immediately implies an algorithm for the distinguishing task in Definition~\ref{def:distinguish}.

\begin{lemma}\label{lem:trivial_reduction}
    For $d\in\mathbb{N}$, let $\calC_d$ be any function class from $\R^d\to\R$ for which the indexing functions $f^{[j]}$, given by $f^{[j]}(x) = x_j$ for some $j\in[d]$, are elements of $\calC$. Suppose that for any $d_1,d_2 = \poly(d)$, there is a $\poly(d)$-time algorithm $\calA$ for the following task. Let $d_1,d_2=\poly(d)$, and let $\calS$ be a known set of functions $f: \R^{d_1} \to \R^{d_2}$ whose output coordinates are all elements of $\calC_{d_1}$ and such that for any $f\in\calS$, $\dtv(f(\calN(0,\Id_{d_1})),\calN(0,\Id_{d_2})) \ge 1/2$. Then $\calA$ can distinguish with constant advantage whether it is given $\poly(d)$ samples from $f(\calN(0,\Id_{d_1}))$ for some $f\in\calS$ versus samples from $\calN(0,\Id_{d_2})$. 
    
    Under this hypothesis, there is a $\poly(d)$-time algorithm that solves the distinguishing problem of Definition~\ref{def:distinguish} to constant advantage.
\end{lemma}

\begin{proof}
    Note that in situation 1) of Definition~\ref{def:distinguish}, the joint distribution over $(x,y)$ is given by the pushforward $f(\calN(0,\Id))$ where $f:\R^{d+1}\to\R^{d+1}$ is as follows: the first $d$ output coordinates are given by the $d$ indexing functions $f^{[1]},\ldots,f^{[d]}$, and the last output coordinate is given by $h$. By taking $\calS$ in the hypothesis to consist of such $f$, we can thus apply the algorithm $\calA$ to distinguish between the two situations in Definition~\ref{def:distinguish} to constant advantage.
\end{proof}

\noindent Note that the contrapositive of the above lemma implies that any lower bound for the task in Definition~\ref{def:distinguish} immediately implies a lower bound for learning pushforwards. While the aforementioned lower bounds of \cite{chen2022hardness,daniely2021local}, which apply when $\calC_d$ is the family of neural networks with at least two hidden layers and polynomially bounded size and weights, do not show hardness for the task in Definition~\ref{def:distinguish}, note that hardness for this task immediately implies hardness for PAC learning $\calC_d$ from Gaussian examples. Indeed, given an algorithm $\calA$ that, given $(x_1,h(x_1)),\ldots,(x_N,h(x_N))$, outputs a predictor $\wh{h}$ for which $\E[g]{(h(g) - \wh{h}(g))^2}$ is small, one can easily solve the task in Definition~\ref{def:distinguish} by running $\calA$ and estimating the square loss of the predictor from some fresh samples. In situation 2) of Definition~\ref{def:distinguish}, because the labels are random, no predictor can achieve low square loss. So the algorithm which outputs $\mathsf{YES}$ if and only if the empirical square loss on fresh samples is small will distinguish between the two situations with constant advantage.

In other words, showing hardness of Definition~\ref{def:distinguish} for $\calC_d$ would be a \emph{stronger result} than what is already shown in \cite{chen2022hardness,daniely2021local}. Putting this and Lemma~\ref{lem:trivial_reduction} together, we conclude that even this stronger hardness result would only imply hardness for learning pushforwards given by $f$ whose output coordinates are functions in $\calC_d$ given by neural networks with at least two hidden layers and polynomially bounded size and weights. In contrast, in the present work, we show hardness for \emph{one} hidden layer, \emph{logarithmic} size, and polynomially bounded weights.

\end{document}